\documentclass{amflinkedin}

\usepackage[utf8]{inputenc}
\usepackage{amsmath,amssymb,amsthm,mathtools}
\newtheorem{proposition}{Proposition}[section]
\usepackage{array}
\usepackage{float}
\floatstyle{ruled}
\newfloat{algorithm}{tbp}{loa}
\floatname{algorithm}{Algorithm}
\floatstyle{plain}
\usepackage{url}

\definecolor{posblue}{HTML}{315B75}
\definecolor{semorange}{HTML}{315B75}
\definecolor{repairgreen}{HTML}{7B8874}
\definecolor{invalidred}{HTML}{9A766B}
\definecolor{mutedgray}{HTML}{7A858C}
\newcolumntype{R}[1]{>{\raggedright\arraybackslash}p{#1}}
\newcommand{\minimal}{\textsc{Shift}}
\newcommand{\exactrepair}{\textsc{Replay}}
\newcommand{\sur}{\textsc{SUR}}
\newcommand{\dirty}{\textsc{Dirty}}
\newcommand{\full}{\textsc{Re-prefill}}

\title{\parbox[t]{\dimexpr\linewidth-2.8cm\relax}{\raggedright Budgeted KV-Cache Repair\\for Evolving Contexts}}
\author[\clubsuit]{Han Yu}
\author[\clubsuit]{Wenhui Zhu}
\author[\clubsuit]{Jincheng Cao}
\author[\clubsuit]{Rui Cai}
\author[\clubsuit]{Yuanda Xu}
\author[\clubsuit]{Zhengze Zhou}
\author[\diamond]{Xiwen Chen}
\author[\ddagger]{Hao Wang}
\author[\ddagger]{Xuanzhao Dong}
\author[\triangle]{Han Shi}
\author[\clubsuit]{Menglin Zhou}
\author[\clubsuit]{Shayan Mohajer Hamidi}
\author[\clubsuit]{Lin Su}
\author[\clubsuit]{Mina Razghandi}
\author[\clubsuit]{Sayan Ghosh}
\author[\clubsuit]{Qing Lan}
\author[\clubsuit]{Rohit Jain}
\author[\clubsuit]{Alborz Geramifard}

\affiliation[\clubsuit]{LinkedIn}
\affiliation[\diamond]{Clemson University}
\affiliation[\ddagger]{Arizona State University}
\affiliation[\triangle]{Rice University}
\correspondence{Han Yu (\email{hanyu@linkedin.com})}

\hypersetup{pdfauthor={Han Yu, Wenhui Zhu, Jincheng Cao, Rui Cai, Yuanda Xu, Zhengze Zhou, Xiwen Chen, Hao Wang, Xuanzhao Dong, Han Shi, Menglin Zhou, Shayan Mohajer Hamidi, Lin Su, Mina Razghandi, Sayan Ghosh, Qing Lan, Rohit Jain, Alborz Geramifard}}

\hypersetup{
  pdftitle={Budgeted KV-Cache Repair for Evolving Contexts},
  pdfsubject={Budgeted KV-cache repair and selective recomputation},
  pdfkeywords={KV cache, context editing, selective recomputation, language models}
}

\let\aftabstract\abstract
\RenewDocumentEnvironment{abstract}{+b}{\gdef\arxivabstracttext{#1}}{}
\begin{abstract}
Language-model agents revise their working context as information becomes obsolete, but deleting text does not erase its influence on cached computation. Later key--value (KV) states remain conditioned on the deleted content, while recomputing the entire retained suffix can be costly. We introduce budgeted KV-cache repair: recovering task performance by recomputing a user-chosen fraction of the retained suffix. Comparing repair allocations under identical relevance scores and token budgets, we find that refreshing coherent regions can substantially outperform refreshing isolated tokens. This finding motivates Semantic-Unit Refresh (\sur{}), which allocates the budget to query-relevant text regions using inexpensive embedding scores, without training or an auxiliary transformer forward. To execute these repairs efficiently, a layer-packed executor processes selected states together at each layer while preserving full causal cache access. Across five models, \sur{} achieves strong recovery in controlled deletion tasks while recomputing only 2--5\% of the suffix. On four natural multi-hop QA datasets, it achieves near-full-recompute exact match at a 5\% budget with lower post-edit latency. The executor accelerates fragmented refresh by up to 16.1$\times$ over interval-wise execution.
\end{abstract}

\aftabstract{\arxivabstracttext}

\begin{document}
\maketitle
\section{Introduction}
\label{sec:introduction}

Language-model agents build their working context from tool outputs, retrieved passages, and intermediate decisions \citep{yao2023react,packer2023memgpt}. As a task progresses, some of this information becomes obsolete: an observation is superseded, a passage is replaced, or an abandoned branch of reasoning is discarded. Removing such content keeps the context relevant to subsequent decisions, but creates a problem for cache reuse. A key--value (KV) cache stores representations computed from the earlier context, not independent encodings of its tokens. Deleting a passage therefore is not guaranteed to remove its influence from the states that follow it \citep{li2026kveraser,li2026notes}.

This residual influence makes a small text edit potentially expensive. The prefix before the deletion remains valid, whereas the retained suffix was computed while the deleted content was still visible. Correcting the suffix's positions repairs its coordinates, but not its computational history \citep{he2025pie,ma2026leyline}. Replaying the entire suffix from the valid prefix repairs both, yet makes the cost depend on how much text follows the edit rather than how much was removed. An agent may thus have to revisit a long history after discarding only a short record. The question is whether answering the next query requires all of that recomputation.

Recent work on cache editing and selective reuse suggests that full reconstruction is not always necessary. Learned editors compensate for erased context, while CacheBlend and ProphetKV select states for recomputation using cache deviation or query attention \citep{li2026kveraser,yao2025cacheblend,wang2026prophetkv}. This raises a practical question for edited contexts: given a fixed amount of recomputation, where should it be spent to recover the next answer? We study this question as \emph{budgeted KV-cache repair}, letting the user choose the fraction of the retained suffix to refresh. The objective is to recover answer quality within that allowance, rather than reconstruct every state of a clean cache.

We investigate this allocation problem empirically, comparing relevance signals and ways of grouping the positions to refresh. A central finding is that relevance alone does not determine repair quality: with identical relevance scores and token budgets, refreshing coherent regions can recover substantially more answers than refreshing isolated high-scoring tokens. The useful unit of repair can thus extend beyond the tokens that make a passage appear relevant.

This finding motivates Semantic-Unit Refresh (\sur{}), which distributes the repair budget across query-relevant regions (Figure~\ref{fig:repair-locality-sur-pipeline}). Frozen input embeddings provide inexpensive relevance scores, while text boundaries define regions to refresh together. The resulting policy can cover separated evidence without training or an auxiliary transformer forward.

\begin{figure*}[t]
  \centering
  \includegraphics[width=0.95\textwidth]{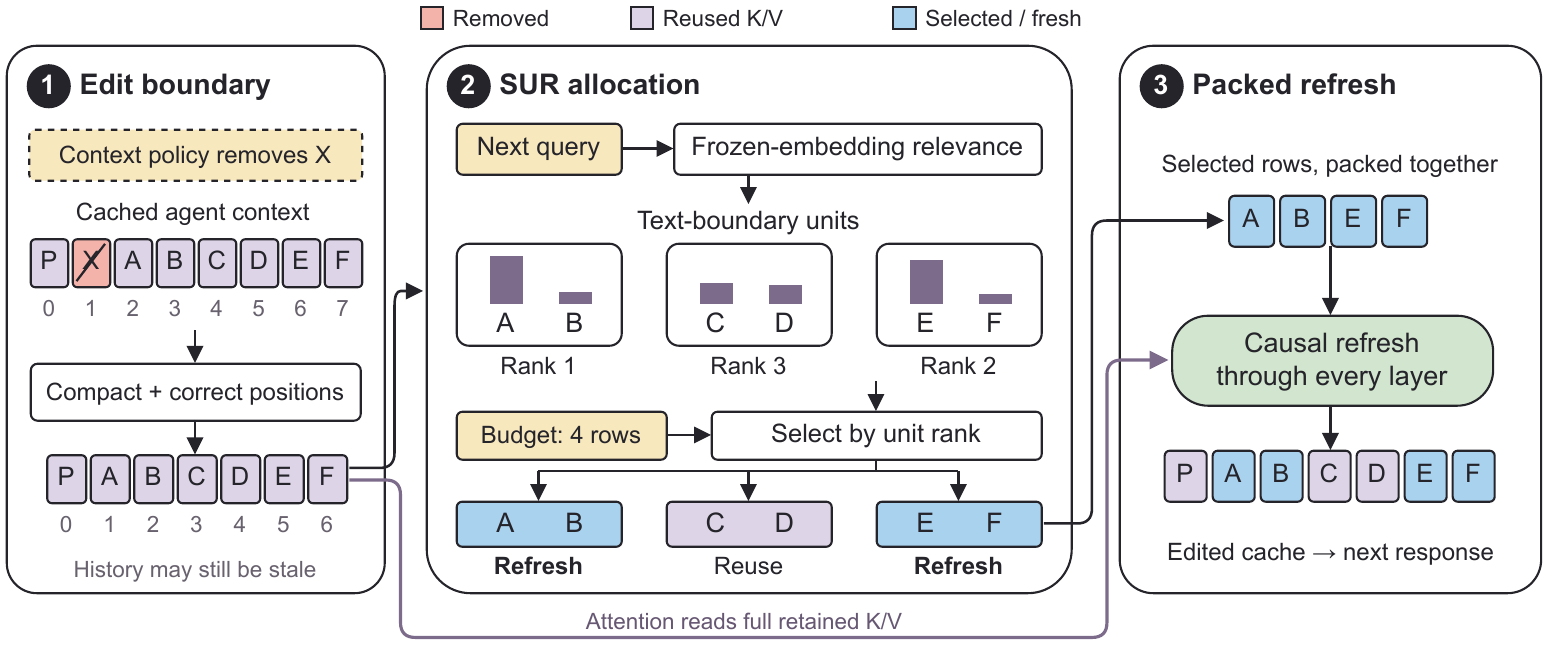}
  \caption{\textbf{Budgeted repair at an agent's context-edit boundary.} After deletion and position correction, \sur{} scores retained tokens against the next query using frozen embeddings, ranks text-boundary units by aggregated relevance, and selects regions under a user budget. Lower-scoring neighbors can be refreshed with their selected unit. Packed refresh updates selected states while retaining full causal cache access. Relevance bars and the four-row allowance are schematic.}
  \label{fig:repair-locality-sur-pipeline}
\end{figure*}

Region allocation also creates an execution challenge: refreshing separated regions one at a time repeatedly traverses the model. We develop a layer-packed executor that processes the selected regions together, reducing fragmented-refresh overhead while preserving access to the full causal cache.

Our contributions are:
\begin{itemize}
  \item \textbf{Grouping changes what the same repair budget can recover.} Holding token scores and budgets fixed, region allocation improves exact match by over 80 percentage points over individual token selection on Qwen3-8B and Llama-3.3-70B at a 2\% budget. These controls identify support organization as a design choice distinct from relevance scoring.
  \item \textbf{Low-cost region selection enables small repairs.} Across five models, \sur{} reaches 89.6--100\% exact match with only 2\% suffix recomputation in controlled deletion tasks. On four natural multi-hop QA datasets, a 5\% budget yields near-full-recompute exact match with lower post-edit latency.
  \item \textbf{Dispersed repairs share one model traversal.} Our packed executor accelerates fragmented refresh by up to 16.1$\times$ over interval-wise execution, preserving full causal cache access while amortizing traversal overhead across regions.
\end{itemize}

\section{Related work}
\label{sec:related}

\paragraph{Context management and cache editing.}
Context-management methods determine what information a model retains. MemGPT organizes memory across storage tiers, while MEMENTO trains models to compress reasoning blocks into summaries \citep{packer2023memgpt,kontonis2026memento}. At the cache level, H$_2$O and SnapKV retain attention-important states to reduce memory and decoding costs \citep{zhang2023h2o,li2024snapkv}. Retaining useful information, however, differs from updating computation after that information changes: a surviving state can still encode a removed passage. Our interface takes the deletion as given and budgets fresh computation over the retained suffix; unselected states remain available to attention.

Cache-editing work addresses this state-maintenance problem more directly. PIE corrects rotary positions after code edits, and Leyline exposes policy-directed splicing alongside a re-prefill mode for semantic forgetting \citep{he2025pie,ma2026leyline}. Beyond positional correction, KVEraser learns replacement steering states for the erased interval \citep{li2026kveraser}, while Models Take Notes identifies downstream representations of earlier conclusions and studies selective recomputation and append-only correction \citep{li2026notes}. These approaches establish that an edit's effects extend beyond its own cache entries. We investigate how to distribute a user-chosen recomputation budget after physical span removal, recovering task quality through selective recomputation without a learned editor or corrective text.

\paragraph{Selective recomputation and budget allocation.}
Modular cache reuse avoids rebuilding recurring context: Prompt Cache provides reusable prompt modules, while subsequent systems reconstruct cross-chunk interactions selectively \citep{gim2024promptcache}. CacheBlend uses KV deviation, Cache-Craft uses chunk-dependence metadata, and EPIC's LegoLink recomputes chunk-initial tokens \citep{yao2025cacheblend,agarwal2025cachecraft,hu2025epic}. CacheTune derives salience from frequency-domain KV structure; SpecCache combines auxiliary-model features with score-based span completion \citep{li2026cachetune,wen2026speccache}. These mechanisms address contextual mismatch when cached pieces are assembled. Deletion creates a related mismatch within an already contextualized suffix. We compare adapted selection signals under a shared deletion-repair operator and equal suffix budgets, separating allocation decisions from the original systems' serving implementations (Appendix~\ref{app:protocols}).

Query-aware reconstruction further directs computation toward the next answer. A\textsuperscript{3} uses shallow query attention, ProphetKV aggregates query attention across layers, and InfoFlow KV aligns prompt-conditioned selection with inference-time positional geometry \citep{zhou2025a3,wang2026prophetkv,teng2026infoflow}. QCFuse reduces probing overhead through compact chunk anchors and critical-layer profiling \citep{yan2026qcfuse}. For budgeted deletion repair, the question is not only which tokens score highly, but which surrounding computation should be refreshed with them. Our matched-score experiments isolate this allocation choice: coherent regions can recover answers that isolated high-scoring tokens miss at the same budget. \sur{} turns this finding into a training-free policy: frozen embeddings provide inexpensive relevance, and text boundaries organize the budget into coherent refresh regions.

\section{Budgeted cache repair}
\label{sec:background}
\label{sec:framework}

Budgeted repair separates how much to recompute from where to spend that computation. Given a context deletion, the next query, and a user-specified budget, we first correct cache positions, use \sur{} to select coherent regions, and refresh their states while reusing the rest of the cache (Figure~\ref{fig:repair-locality-sur-pipeline}). Packed execution batches the selected regions to avoid repeated model traversals.

\subsection{The budgeted repair problem}
\label{sec:common-repair-operator}
\label{sec:budget-contract}

We consider a cached token sequence comprising a retained prefix $p$, a span $e$ to delete, and a retained suffix $s$. With $\oplus$ denoting concatenation, the original context is $x=p\oplus e\oplus s$ and the edited context is $x^-=p\oplus s$. Let $q$ denote the next query and $S=|s|$ the number of retained suffix tokens.

\begin{figure*}[t]
\centering
\includegraphics[width=0.96\textwidth]{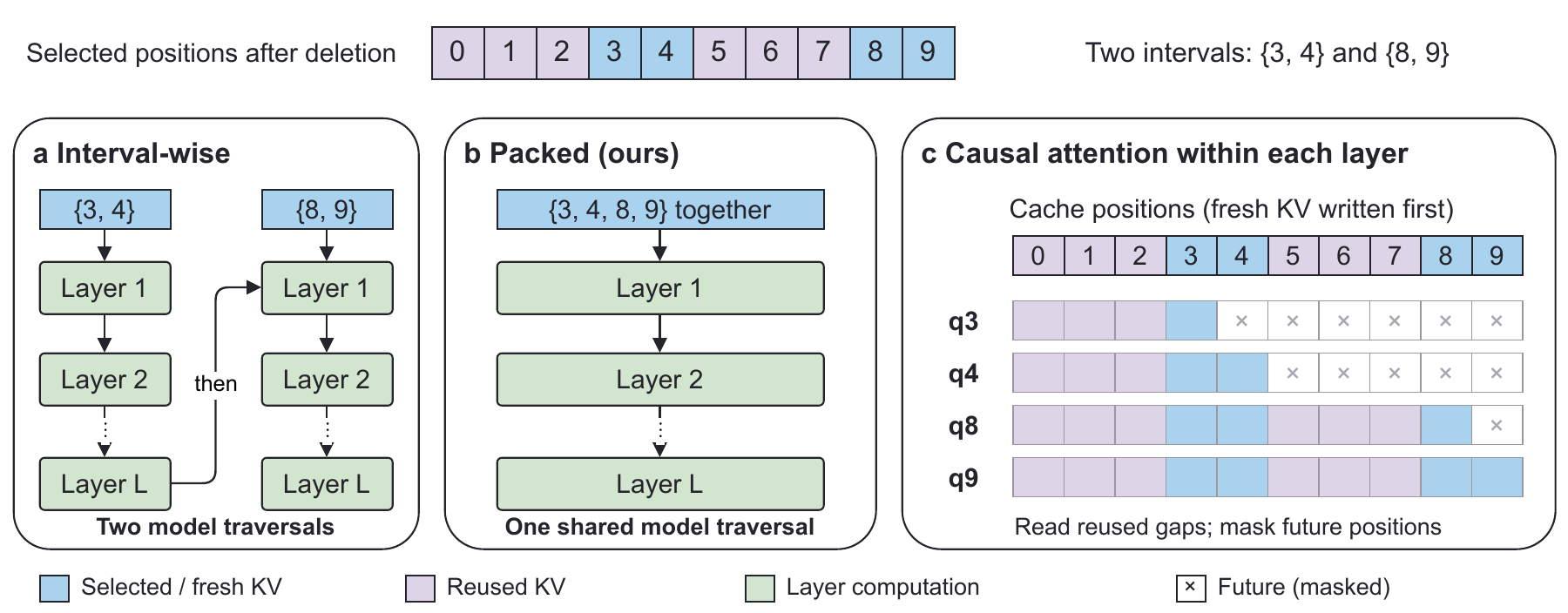}
\caption{\textbf{Same selected tokens, a different execution order.} (a) Interval-wise refresh completes every layer for positions $\{3,4\}$ before traversing the model again for $\{8,9\}$. (b) Packed refresh processes all four positions together at each layer, sharing one model traversal. (c) After selected K/V are updated, attention still reads all causally available entries, including reused gaps. Filled cells indicate visible K/V, not attention weights; crosses mark future positions.}
\label{fig:packed-execution}
\end{figure*}

The edit leaves the prefix cache valid because those states cannot depend on later tokens. The suffix is different: its states were computed while the removed span was still available to attention. Removing $e$ also moves the suffix to earlier positions. We therefore delete the corresponding cache entries and correct the suffix's rotary positions \citep{su2021roformer,he2025pie}, but neither operation recomputes the suffix representations. Recomputing $s$ under the edited history would rebuild them, at a cost that grows with the entire retained suffix.

Our objective is to recover answer quality for $q$, rather than reproduce every state of a newly built cache. The user supplies a recomputation budget, expressed as a percentage $b$ of the $S$ suffix tokens. This gives an allowance of $B$ positions; the selection rule chooses the set $A$ of $B$ positions to refresh. No retained text is discarded: the budget changes which states receive new computation, not which tokens attention can access. Appendix~\ref{app:operator-semantics} specifies the integer rounding.

Given $A$, we recompute the selected tokens through every transformer layer and replace their cached keys and values. All other entries are reused and remain available under causal attention. Sparse repair is approximate because selected tokens can still read earlier unrefreshed states computed with $e$ present (Appendix~\ref{app:mixed_cache_error}). A zero budget leaves only deletion and position correction; a full budget recomputes the entire suffix. Between these endpoints, the user can trade recomputation for answer recovery while retaining the same context.

With the refresh operation fixed, the design problem is to choose $A$ from the retained context and the next query. Appendix~\ref{app:task-aware-allocation} formalizes this objective. We next develop a selection rule that allocates the budget to query-relevant regions.

\subsection{Semantic-Unit Refresh}
\label{sec:allocator-families}

A relevance score identifies promising locations, but does not determine which surrounding tokens should be refreshed with them. Our matched-score experiments show that coherent regions can recover answers missed by isolated high-scoring tokens at the same budget (Section~\ref{sec:allocator-controls}). \sur{} turns this observation into a low-cost policy: text boundaries supply candidate regions, and query relevance determines which regions receive computation.

\paragraph{Scoring relevance without a model forward.}
The selection signal should be inexpensive relative to the computation it saves. We therefore use the model's frozen input embedding table $E$, without an auxiliary transformer forward or a trained selector. For suffix token $s_i$ and query token $q_j$, the score is the strongest embedding match to the query, with a bonus for an identical token:
\begin{equation}
 \begin{aligned}
 r_i={}&\max_{1\leq j\leq |q|}
 \left\langle \frac{E(s_i)}{\|E(s_i)\|_2},
                 \frac{E(q_j)}{\|E(q_j)\|_2}\right\rangle\\
 &+0.25\,\mathbf{1}[s_i\in q].
 \end{aligned}
 \label{eq:sur-relevance}
\end{equation}
Here $\mathbf{1}[s_i\in q]$ indicates whether the suffix token occurs in the query. These scores supply the relevance signal; the next step determines how to distribute the budget around it.

\paragraph{Allocating the budget to regions.}
We group tokens at message, structured-container, paragraph, and sentence boundaries, splitting long units and merging short ones within their container. Each unit receives the mean of its four highest token scores, or all scores if shorter. A few relevant tokens can thus identify a unit for refresh, bringing its lower-scoring neighbors into the same computation. We rank units by this score, with earlier positions breaking ties.

To fill the allowance $B$, we admit whole units in ranked order. If the next unit does not fit, we select a contiguous window within it that uses exactly the remaining allowance. We choose the window with the largest summed forward-decaying relevance envelope, which extends high relevance toward following tokens. Appendix~\ref{app:sur-spec} specifies this scoring rule and the segmentation details. The resulting set $A$ contains exactly $B$ positions and may cover several separated regions. We sort these positions, merge adjacent ones into execution intervals, and refresh the same selection at every layer (Algorithm~\ref{alg:sur}).

\subsection{Packed execution}
\label{sec:packed-execution}

Packed execution avoids traversing the model separately for each selected region. The regions chosen by \sur{} can form $R$ disjoint intervals. An interval-major schedule processes the first interval through every layer, then repeats the traversal for the next interval, incurring dispatch and cache-management work $R$ times. A small token budget can therefore still require many model traversals when the selection is fragmented.

We instead use a layer-major schedule: all selected positions advance through one layer together before moving to the next (Figure~\ref{fig:packed-execution}). We gather their token embeddings into a compact batch and keep their hidden states together throughout the traversal. At each layer, we compute fresh K/V for these positions, write them into the retained cache, and evaluate attention and the remaining layer operations for the selected tokens. This produces the packed hidden states for the next layer and, after the final layer, a cache refreshed at exactly the selected positions.

Only the recomputation is packed; the context available to attention is not. Each selected token keeps its position in the edited context and attends to the full causal history: fresh K/V at selected positions and reused K/V at unselected positions, including gaps between selected regions. Its position determines which entries are visible, so future tokens remain masked. Packing thus changes the order of computation without restricting attention to the packed tokens or changing the selected set. Appendix~\ref{app:kernel-execution} gives the implementation and the equivalence argument.

\label{sec:cost-model}
Selection and packing therefore reduce different costs. Across $L$ transformer layers, selecting $B$ of $S$ suffix positions reduces token-layer updates from $LS$ to $LB$: a 2--5\% budget removes approximately 95--98\% of suffix-position recomputation. Packing executes that same selection in one model traversal rather than $R$, batching the selected computation and amortizing traversal overhead across regions. Appendix~\ref{app:recompute-scaling} develops the arithmetic and overhead models.

\section{Experiments}
\label{sec:results}

We evaluate three aspects of budgeted repair. Controlled deletion tests answer recovery and the effect of support organization under matched relevance scores and budgets. Natural multi-hop QA tests whether \sur{} offers a useful quality--latency tradeoff on real questions. A matched-work execution benchmark isolates how packing reduces fragmented-refresh overhead. Section~\ref{sec:discussion} connects these results to allocation and budget choice.

\begin{figure*}[t]
\centering
\includegraphics[width=0.95\textwidth]{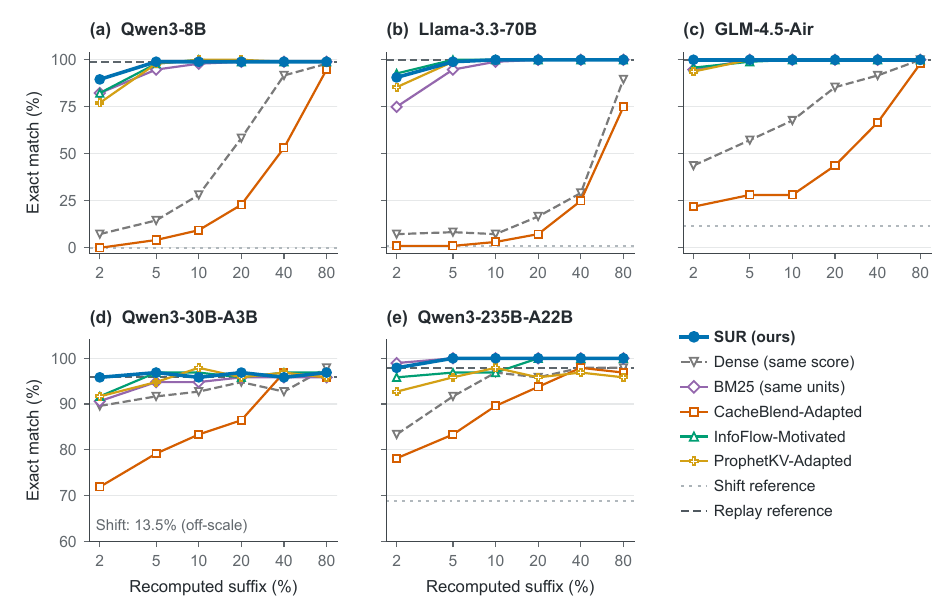}
\caption{\textbf{Budgeted repair across five model systems.} Every panel uses the same 96 cases and baseline set, including already-correct cases. Lines connect measured suffix budgets; Shift and Replay are reference operators. Prior-derived selectors share the same repair operator (full results in Appendix~\ref{app:budget-frontiers}). The bottom row uses a magnified vertical range; Qwen3-30B's Shift score is annotated off-scale.}
\label{fig:deletion-matrix-response}
\end{figure*}

\begin{figure*}[t]
\centering
\includegraphics[width=0.95\textwidth]
{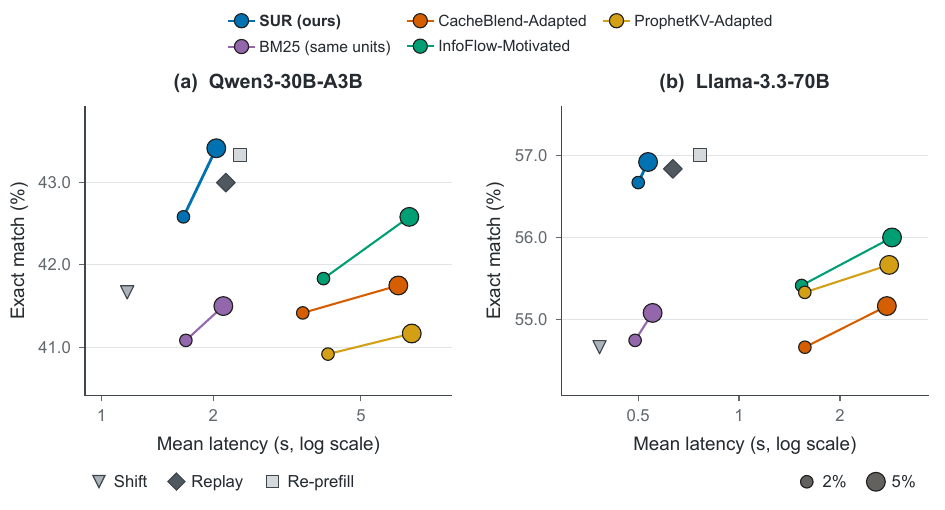}
\caption{\textbf{Answer quality versus repair-and-generation latency.} Exact match across four datasets (1,200 cases from 400 questions) versus mean post-edit latency, including selection and decoding. Small and large markers denote 2\% and 5\% recomputation. All methods use the same sequential executor. Paired quality comparisons appear in Table~\ref{tab:naturalqa-paired-ci}.}
\label{fig:natural-qa-main}
\end{figure*}

\subsection{Experimental setup and baselines}
\label{sec:experimental-protocol}

\paragraph{Models and common protocol.}
Our evaluation covers Qwen3-8B, Qwen3-30B-A3B, Llama-3.3-70B, GLM-4.5-Air, and Qwen3-235B-A22B \citep{yang2025qwen3,dubey2024llama3,glm45team2025}. We use BF16 weights, greedy decoding, and each model's native tokenizer and chat template. For a given case and budget, every allocator selects exactly the same number of suffix positions for all-layer refresh. We measure task quality and post-edit answer latency with the shared sequential interval executor; Section~\ref{sec:execution-results} evaluates packed refresh timing.

\paragraph{Reference operators.}
Four paths anchor the comparison. \dirty{} keeps the deleted span and is an untreated reference. \minimal{} removes it and corrects suffix positions without contextual recomputation. \exactrepair{} retains the valid prefix and recomputes the entire suffix; \full{} rebuilds the complete edited prompt from scratch. Shift and Replay provide the zero- and full-suffix-budget endpoints.

\paragraph{Allocation controls.}
Four controls distinguish scoring from grouping. \emph{Dense tokens} selects the highest-scoring individual tokens using \sur{}'s relevance scores. \emph{Max-relevance contiguous} uses the same scores but selects the single highest-sum window that fills the budget. \emph{BM25 (same units)} substitutes lexical relevance for embedding relevance while retaining \sur{}'s units and allocator. \emph{Random semantic units} selects those units in random order, testing whether grouping alone suffices without query direction.

\paragraph{Prior-work-derived allocators.}
We compare eight further policies by transferring their allocation signals into the common deletion-and-refresh operator. \emph{ProphetKV-Adapted} ranks tokens by mean query attention, whereas \emph{InfoFlow-Motivated} uses the attention tensor's $L_2$ norm \citep{wang2026prophetkv,teng2026infoflow}. \emph{Cache-Craft-Adapted} measures retained tokens' attention to deleted keys; \emph{A\textsuperscript{3}-Adapted} uses query attention after a shallow edited-context prefill \citep{agarwal2025cachecraft,zhou2025a3}. \emph{CacheBlend-Adapted} instead ranks shallow KV deviations from the position-corrected cache \citep{yao2025cacheblend}.

Other policies use chunk geometry or representation magnitudes. \emph{LegoLink-Adapted} prioritizes positions near chunk boundaries, \emph{CacheTune-Adapted} scores low-frequency reconstructed KV magnitudes, and \emph{SpecCache-Adapted} uses auxiliary-model hidden-state norms with span selection \citep{hu2025epic,li2026cachetune,wen2026speccache}. A separate \emph{clean-state oracle} ranks deviations from a fully recomputed cache; it requires full recomputation to obtain its signal. Together with \sur{} and the four controls, these give fourteen allocators. Their comparison concerns allocation under a shared repair operator, not the original systems' serving throughput. Appendix~\ref{app:deletion-matrix-protocol} specifies the adaptations and selection rules.

\subsection{Controlled answer recovery on Erasing-NIAH}
\label{sec:primary-response}

We first ask how much of the retained suffix must be refreshed when deletion changes the correct answer. Erasing-NIAH \citep{li2026kveraser} embeds two same-key facts in a long context; deleting one changes the required ordered list of values. This construction makes the effect of deletion directly checkable, allowing us to compare allocation rules over a broad budget range.

All five models evaluate the same 96 cases spanning 1K--32K contexts. We sweep all fourteen allocators at 2/5/10/20/40/80\% suffix budgets, alongside the four reference operators. Exact match (EM) requires the predicted list to match the edited-context answer in both order and multiplicity. Every score includes all 96 cases, including those already correct under Shift. Appendix~\ref{app:common96-core-protocol} gives the construction and answer-scoring details.

\paragraph{Recovery with small budgets.}
Figure~\ref{fig:deletion-matrix-response} shows that high answer accuracy can be recovered by refreshing only a small fraction of the stale suffix. At 2\% recomputation, \sur{} reaches 89.6--100\% EM across the five systems. Even the largest model, Qwen3-235B, benefits substantially over Shift despite its stronger starting point: model scale alone does not eliminate repair demand.

Increasing the budget brings several models to their dense-reference performance without refreshing the entire suffix. This does not imply a universal saturation point: Qwen3-30B stays near Replay but below Re-prefill across the sweep. The budget offers a tunable tradeoff, with the useful operating point depending on the model.

\subsection{Quality and latency on natural multi-hop QA}
\label{sec:natural-breadth}

The controlled task establishes answer recovery; we next ask whether budgeted repair can preserve accuracy on real questions while reducing the cost of answering them. We use 2WikiMultiHopQA, HotpotQA, IIRC, and MuSiQue to test this tradeoff beyond the same-key numeric construction \citep{ho2020twowiki,yang2018hotpotqa,ferguson2020iirc,trivedi2022musique}. We select 100 source questions per dataset before observing model outcomes. For each source, inserting the same non-supporting passage before, between, or after supporting evidence gives three placement cases. Deleting that passage restores the clean document and its target answer, but leaves suffix states computed with the passage present. Each model has 1,200 cases, with 300 per dataset.

We compare a mixture-of-experts model, Qwen3-30B-A3B, and a dense model, Llama-3.3-70B. At 2\% and 5\% budgets, \sur{} is evaluated against same-unit BM25 and the CacheBlend-, InfoFlow-, and ProphetKV-derived selectors, with Shift, Replay, and Re-prefill as references. This comparison includes both inexpensive lexical selection and forward-based selection signals, whose acquisition cost matters when the repair budget is small.

\paragraph{Quality and timing.}
We report normalized answer EM over all cases and mean post-edit latency, including selection, position correction, refresh, and query/answer decoding. Refresh includes prefix copying and cache assembly; initial dirty-context prefill and harness-only clones are excluded, while Re-prefill's edited-context rebuild is included. Methods share the same worker and harness for each case, with one timed trial per case and no warmup. Quality differences use source-paired 95\% bootstrap intervals, keeping each question's three placements together. Appendices~\ref{app:natural-detail}--\ref{app:naturalqa-paired-ci} give hardware settings, per-dataset results, and the statistical protocol.

\paragraph{The quality--latency tradeoff.}
Figure~\ref{fig:natural-qa-main} shows that small-budget repair can preserve answering accuracy while reducing post-edit latency on real questions. At a 5\% budget, \sur{} is within 0.1 percentage points of Re-prefill's observed EM on both models and is 1.16--1.43$\times$ faster. It also has lower latency than Replay. The smaller budget trades some quality for further savings, most visibly on Qwen3-30B.

The comparison with other selectors highlights the importance of inexpensive allocation. At the larger budget, \sur{} has both higher EM and substantially lower latency than the attention/deviation-derived baselines on both models. Same-unit BM25 offers similarly low selection overhead; on Llama, it is marginally faster at the smaller budget but less accurate. Thus, acquiring a more expensive selection signal does not necessarily improve the quality--latency tradeoff.

Source-paired analysis supports contextual refresh over position correction alone: at 5\%, the improvements over Shift have positive confidence intervals on both models. The small differences from Re-prefill have intervals spanning zero (Table~\ref{tab:naturalqa-paired-ci}), consistent with the near-full-recompute observed EM.

\subsection{Execution efficiency of packed refresh}
\label{sec:execution-results}

Finally, we isolate fragmentation overhead at a fixed budget. On Qwen3-8B (BF16, batch one, H100 80GB), we fix $B=64$ and vary $R=1,4,16$ at 2K and 8K context lengths, matching causal query--key pair counts within each length. We compare packed scaled dot-product attention (SDPA) with interval-wise, direct-cache sequential, and existing one-shot sparse execution. These controls separate shared-traversal savings from cache-assembly and adapter overhead.

\paragraph{Matched-work refresh timing.}
Refresh time includes all-layer recomputation, cloning, planning, and mask construction, excluding selection, position correction, and answer generation. We report medians over 20 interleaved trials after ten warmups. This repeated benchmark measures the recomputation stage rather than the complete-answer latency reported for Natural QA.

Packing keeps refresh time nearly flat as the same recomputation budget becomes more fragmented, yielding up to 16.1$\times$ faster refresh than interval-wise execution. Most of this gain comes from sharing model traversal: the existing one-shot path already captures much of the benefit, with the packed adapter adding a further 1.21$\times$ at 8K with sixteen intervals. For a single interval, direct-cache sequential execution is slightly faster than packing. The useful execution strategy therefore depends on selection geometry, not just token count.

Full refresh timings and backend controls appear in Appendix~\ref{app:kernel-execution}, with supplementary recovery timings on actual QA-selected regions in Appendix~\ref{app:qa-recovery-timing}.

\section{Discussion}
\label{sec:discussion}
\label{sec:allocator-controls}

The Erasing-NIAH allocation controls identify support organization as a design choice distinct from relevance scoring. With identical token scores, budgets, and repair operators, Qwen3-8B's EM at 2\% rises from 7.3\% for individual tokens to 75.0\% for the best contiguous window and 89.6\% for \sur{}; Llama exhibits the same progression. The window already maximizes summed relevance over all windows of that width, yet distributing computation across relevant regions improves recovery further. Random unit selection performs poorly, while same-unit BM25 and InfoFlow-Motivated win at some operating points (Appendix~\ref{app:design-analysis}). The resulting principle is query-directed region allocation. \sur{} implements it with inexpensive text boundaries; whether linguistic boundaries improve on generic multi-window grouping remains open.

This dependence on allocation changes how the budget should be interpreted. On Qwen3-8B, \sur{} reaches near-perfect EM at 5\%, while same-score token selection approaches that level at 80\% in the measured sweep. A budget therefore specifies how much computation is available, not how much accuracy it will recover. Natural QA complements this controlled finding by testing \sur{}'s practical operating points: a 5\% budget approaches Re-prefill's observed EM while reducing post-edit latency. The highlighted budgets are useful choices, not limits on the interface; the allowance can be increased as the task's quality requirements demand.

Turning reduced computation into faster refresh requires an executor suited to the selection. The matched-work results motivate direct execution for a single interval and layer packing for fragmented selections, avoiding repeated traversals of the model. Allocation and execution thus address two parts of the same objective: reducing the suffix computation needed for recovery and reducing the overhead of carrying it out.

\section{Conclusion}

Budgeted KV-cache repair makes recomputation after a context edit a task-directed allocation problem. Our controlled experiments show that allocating the same token budget to relevant regions can recover answers missed by isolated-token refresh. \sur{} translates this finding into a low-cost policy, with strong controlled recovery at 2--5\% budgets and a favorable quality--latency tradeoff on natural multi-hop QA. Layer-wise packing separately accelerates fragmented refresh by up to 16.1$\times$ over interval-wise execution. Together, these results make region allocation and execution geometry central considerations in maintaining reusable context caches.

\label{main-content-end}

\clearpage
\section*{Limitations}
\label{sec:limitations}

\paragraph{Execution coverage.}
The execution study uses one Qwen3-8B checkpoint on one H100, with fixed-budget fragmentation sweeps at two context lengths and four real-QA timing inputs. These stage-specific measurements cannot establish a deployment-wide speedup. Layer packing changes low-precision execution paths: tested next-token argmax agreement is not a substitute for rerunning task-level evaluation. The optional indexed attention kernel passes the 8K full-model numerical control but exceeds its fixed error budget for one 2K value-state statistic. It remains experimental; packed SDPA is the default. Dense Llama/Qwen3 with full attention, fixed default rotary positions, batch size one, and a single device define the implemented packed scope; MoE, offloading, sliding windows, and distributed serving are not covered. Preparation overlap is measured on one text example, without generation.

\paragraph{Baseline adaptation.}
Prior-work-derived allocators share our repair operator and budget, which isolates allocation choices but does not reproduce each original system's optimized implementation or full policy. Conclusions about these adapted signals should not be read as native-system rankings.

\section*{Ethical Considerations}

The study concerns reuse of model computation after context editing, not deletion of information from model weights or a privacy guarantee. A repaired cache may retain deleted information in unrefreshed states. Applications requiring confidential-data removal must therefore use a separately validated security mechanism rather than infer erasure from answer accuracy. The experiments use existing research benchmarks; their source citations and evaluation transformations are documented in Appendix~\ref{app:protocols}.

\bibliographystyle{acl_natbib}
\bibliography{references}

\clearpage
\beginappendix
\newcolumntype{L}[1]{>{\raggedright\arraybackslash}p{#1}}

\section{Repair operators, allocation rules, and evaluation}
\label{app:protocols}

This section specifies the cache-maintenance interface, selector implementations, and task evaluators. Repair semantics and selected-token count are fixed across allocators; each rule determines \emph{which} retained positions receive computation.

\subsection{Operator semantics and exact budgets}
\label{app:operator-semantics}
\label{app:response-definitions}

Write the cached pre-query sequence as $x=p\oplus e\oplus s$, with deleted width $d=|e|$ and retained suffix length $S=|s|$.  The requested context is $x^-=p\oplus s$.  At layer $h$, positional correction maps the old suffix key at position $i$ to its compacted position:
\begin{equation}
 \begin{aligned}
 \widetilde K^h_{i-d}&=R^h_{i-d}(R^h_i)^{-1}K^h_i,\\
 \widetilde V^h_{i-d}&=V^h_i.
 \end{aligned}
 \label{eq:appendix-rope-shift}
\end{equation}
Here $R_i^h$ is the model's rotary map, including its position-scaling convention.  This operation corrects position without removing the deleted span's influence on the stored representations.

\paragraph{Four reference paths.}
\dirty{} queries the original cache, with the deleted span still present; it is an untreated reference, not a completed deletion.  \minimal{} physically removes the span and position-corrects the retained suffix, without contextual forwards.  \exactrepair{} retains the valid prefix and recomputes the entire suffix under the edited history.  \full{} rebuilds the complete edited prompt from an empty cache.  Replay and Re-prefill are operational dense references; their floating-point execution paths need not produce identical outputs (Appendix~\ref{app:numerical-systems}).

\paragraph{Sparse refresh reads mixed fresh and stale history.}
Every sparse arm starts from the Shift cache.  Selected positions are recomputed through all transformer layers in causal order and written back; unselected positions retain their shifted KV states.  Consequently, a refreshed position can still read an earlier unselected stale state.  Selecting a position does not establish causal closure.  Consecutive selected positions form execution intervals.  In the task-quality evaluations, the sparse executor uses sequential maximal-interval replay: each maximal contiguous selected interval is forwarded in order against the evolving cache.  Any added bridging positions must count against the same allowance.

For an interval $[a,b)$, the executor copies the current cache prefix $[0,a)$ and forwards the edited tokens $x^-_{a:b}$ against that prefix. The resulting K/V rows replace only $[a,b)$ at every layer; all other rows are retained.  Future cached positions are therefore absent from the forward, while earlier refreshed intervals are available to later ones.  The cache length remains unchanged after each replacement.  This reference implementation uses ordinary model forwards and tensor concatenation, not a specialized sparse kernel; fragmentation incurs repeated prefix-copy and forward overhead.

\paragraph{The allowance is an exact count of unique suffix positions.}
For a requested percentage $b$, every allocator receives
\begin{equation}
 B(S,b)=\min\{S,\lceil Sb/100\rceil\},\qquad 0\leq b\leq100.
 \label{eq:exact-suffix-budget}
\end{equation}
Only pre-query suffix positions are selectable.  The query may supply a ranking signal but cannot consume the repair budget.  We evaluate sparse arms at $b\in\{2,5,10,20,40,80\}\%$; Shift and Replay supply the separate 0\% and 100\% endpoints.  Curves retain nonmonotone observations and do not interpolate an unmeasured success rate.  A percentage is contextual-forward coverage, not a FLOP fraction: sparse arms also rotate/read/write the suffix, and a rotation is not equivalent to an attention--MLP forward.

\subsection{Task-aware allocation and repair locality}
\label{app:task-aware-allocation}

We distinguish two questions: which positions give the best answer quality at a fixed budget, and how little recomputation suffices for a target quality. Fix an edited context $x^-=p\oplus s$ and a nonempty suffix of length $S$; write $[S]=\{1,\ldots,S\}$ for its local indices.

\paragraph{Task utility.}
Let $C_A$ denote the cache obtained by refreshing positions $A\subseteq[S]$, starting from \minimal{}. For the next query $q$, let $\hat y(C,q)$ be the answer decoded from cache $C$. Its task score is
\begin{equation}
    \mathcal{J}_q(C) = \operatorname{Eval}_q\!\left(\hat y(C,q)\right),
\label{eq:task_utility}
\end{equation}
where $\operatorname{Eval}_q$ is the benchmark evaluation rule for that example. In our experiments, this is per-example exact match.

\paragraph{Allocation at a fixed budget.}
With an allowance of exactly $B$ positions, the ideal selection maximizes answer quality:
\begin{equation}
A_B^\star(q) \in \arg\max_{\substack{A\subseteq[S]\\|A|=B}} \mathcal{J}_q(C_A).
\label{eq:ideal_repair_selection}
\end{equation}
This defines the best support in hindsight, not an optimization that \sur{} solves during inference. Scoring a candidate selection requires applying its repair, decoding an answer, and consulting the task evaluator. \sur{} instead selects positions using query-conditioned relevance, without access to their eventual task score.

\paragraph{Minimum sufficient repair.}
For a target quality $\tau$, repair locality is the smallest fraction of suffix positions whose refresh reaches that target:
\begin{equation}
\rho_{\tau}(q) = \frac{1}{S} \min_{A\subseteq[S]}
\left\{|A|: \mathcal{J}_q(C_A)\geq\tau\right\},
\label{eq:repair_locality}
\end{equation}
We set $\rho_\tau(q)=+\infty$ if no selection reaches the target. With exact-match evaluation, $\tau=1$ requires a correct answer. A small value means that little recomputation suffices for the task, without requiring every cache state to match full recomputation. The definition concerns the amount of repair, not whether selected positions are contiguous. Our experiments evaluate selected supports at fixed budgets rather than solving for this minimum.

\subsection{Semantic-Unit Refresh}
\label{app:sur-spec}

\sur{} uses query relevance to rank coherent regions of the edited context. For suffix token $s_i$ and query-token sequence $q$, the frozen embedding table $E$ gives
\begin{equation}
 \begin{aligned}
 r_i={}&\max_{1\leq j\leq |q|}\frac{\langle E(s_i),E(q_j)\rangle}
                         {\|E(s_i)\|_2\|E(q_j)\|_2}\\
      &+0.25\,\mathbf{1}[s_i\in q].
 \end{aligned}
\end{equation}
The query is tokenized without added special tokens or stop-word filtering. Embedding normalization and cosine scoring run in FP32, processing the prompt in 1,024-token blocks.  The selector uses neither a learned router nor an auxiliary generation. Message and structured-container boundaries are preserved; ordinary text is split at sentence and paragraph boundaries.  Token offsets map the resulting units onto the edited prompt: each token is assigned by the midpoint of its character span.  The implementation verifies that tokenization reproduces the edited token sequence and that the units partition the retained suffix without gaps or overlaps.  Units longer than 128 tokens are split into balanced chunks.  A unit shorter than 8 tokens merges with the shorter adjacent unit in its container, when available; an equal-length tie chooses the preceding unit.  Oversized merged units are split again.

Each unit receives the mean of its four largest token scores, or all scores if it contains fewer than four tokens.  Units are ranked by this score, with earlier position breaking ties, and admitted until $B(S,b)$ positions are filled.  If the next unit exceeds the remaining allowance $w$, SUR chooses one width-$w$ window within that unit.  Starting from $u_0=0$, the forward envelope is
\begin{equation}
 \begin{aligned}
 u_t&=\max\bigl\{e^{r_t-r_{\max}},e^{-1/16}u_{t-1}\bigr\},\\
 r_{\max}&=\max_{i\ \mathrm{in\ the\ unit}}r_i.
 \end{aligned}
\end{equation}
The selected window maximizes the sum of $u_t$, with the earliest maximizer breaking ties.  This fills the exact allowance while favoring local forward context; it is a deterministic allocation heuristic, not a dependency oracle.

\begin{algorithm}[t]
\caption{Semantic-Unit Refresh under a suffix budget}
\label{alg:sur}
\small
\textbf{Input:} cache $\mathcal C$ for $p\oplus e\oplus s$, edited tokens $p\oplus s$, query $q$, budget $b$.\par
\begin{tabular}{@{}r@{\quad}p{0.86\linewidth}@{}}
1 & $B\leftarrow\min\{|s|,\lceil |s|b/100\rceil\}$. \\
2 & If $B=|s|$: return full suffix Replay from the valid prefix. \\
3 & $\mathcal C\leftarrow\operatorname{Shift}(\mathcal C,e)$; if $B=0$: return $\mathcal C$. \\
4 & Compute token scores $r_i$ using Equation~\ref{eq:sur-relevance}. \\
5 & Partition suffix into units; rank by their top-four-mean scores. \\
6 & $A\leftarrow\varnothing$; visit units $u$ in ranked order until $|A|=B$: \\
7 & \quad $w\leftarrow B-|A|$; if $|u|\leq w$, add all positions of $u$ to $A$. \\
8 & \quad Otherwise add the best width-$w$ forward-envelope window in $u$. \\
9 & Sort $A$ by position and coalesce adjacent positions into intervals. \\
10 & Refresh selected rows through all layers (Section~\ref{sec:packed-execution}), updating $\mathcal C$. \\
11 & Return $\mathcal C$ for ordinary query processing and answer decoding. \\
\end{tabular}
\end{algorithm}

\subsection{Matched controls and prior-work-derived allocators}
\label{app:deletion-matrix-protocol}

The full NIAH comparison contains fourteen allocators and four references: $14\times6+4=88$ evaluated coordinates per model.  All allocators share Equation~\ref{eq:exact-suffix-budget} and the sparse-refresh operator above. They do not share interval counts, selector overhead, or native serving implementations.

\paragraph{Four controls test support allocation and query relevance.}
\emph{Dense tokens} selects the top $B$ individual positions using the same $r_i$ as SUR. \emph{Max-relevance contiguous} uses those same scores and selects the maximum summed $r_i$ over \emph{all} width-$B$ contiguous suffix windows, choosing the earliest tie. \emph{BM25 (same units)} tokenizes query and unit text into lowercase ASCII-alphanumeric terms and scores SUR's exact units with $k_1=1.2$ and $b_{\mathrm{BM25}}=0.75$; query-term multiplicity weights each term's contribution.  The unit score is broadcast to its tokens and passed through SUR's unit allocator. \emph{Random semantic units} permutes the same units with a case-fixed seed and, when necessary, samples an exact-width window in the last unit. Dense tokens, the single window, and SUR compare support organization with a fixed token score and allowance. Random units test query direction; BM25 tests an alternative relevance rule with the same unit construction and allocator.

\paragraph{Eight prior-work-derived signals use the common repair path.}
Table~\ref{tab:revision-prior-signals} states their actual implementations. \emph{Adapted} transfers a published allocation signal into the deletion setting; \emph{Motivated} indicates a materially different scoring projection.  These rows compare allocation under a shared operator, not the native systems' throughput.  Unless otherwise stated, a score is converted to the exact top-$B$ suffix support with earlier-position tie-breaking.

\begin{table*}[tp]
\caption{\textbf{Prior-work-derived allocation signals.} Layer indices start at zero; layer 1 is the first contextual KV layer.  All rows use the common suffix budget and all-layer refresh operator.}
\label{tab:revision-prior-signals}
\centering\small
\setlength{\tabcolsep}{4pt}
\begin{tabular}{@{}>{\raggedright\arraybackslash}p{0.235\textwidth}
                  >{\raggedright\arraybackslash}p{0.725\textwidth}@{}}
\toprule
Allocator & Implementation in this comparison \\
\midrule
ProphetKV-Adapted~\citep{wang2026prophetkv} & Query-only forward against the Shift cache; suffix score is attention averaged uniformly over query tokens, heads, and layers. \\
InfoFlow-Motivated~\citep{teng2026infoflow} & $L_2$ norm of the same query-to-suffix attention tensor over query tokens, heads, and all layers. This differs from the original attention-mass formula, selected-layer scope, and chunk-reordering procedure. \\
Cache-Craft-Adapted~\citep{agarwal2025cachecraft} & Replay the dirty suffix in 64-token chunks with deleted keys still visible.  Score each retained token by its attention mass on deleted keys, averaged over heads and layers. \\
CacheBlend-Adapted~\citep{yao2025cacheblend} & Fresh two-layer edited prefill; score layer-1 KV deviation from Shift by $\sqrt{\|\Delta K_{1,i}\|_2^2+\|\Delta V_{1,i}\|_2^2}$. One selected support is reused across layers. \\
A\textsuperscript{3}-Adapted~\citep{zhou2025a3} & Fresh two-layer edited prefill followed by a shallow query forward; sum layer-1 query-to-context attention over heads and query positions. \\
LegoLink-Adapted~\citep{hu2025epic} & Treat the valid prefix as the first chunk and split the suffix into 512-token chunks.  Rank positions by depth from their chunk boundary, then chunk index.  Depths beyond the native 32-token limit are extensions required by the common budget grid. \\
CacheTune-Adapted~\citep{li2026cachetune} & Invert compacted-position RoPE; low-pass each layer's suffix K/V in FP32, retaining the lowest $\lfloor0.5(\lfloor S/2\rfloor+1)\rfloor$ rFFT bins.  Score the reconstructed K/V by their mean unsquared $L_2$ norm; sum-normalize each layer's scores and average across layers. \\
SpecCache-Adapted~\citep{wen2026speccache} & A fixed Llama-3.2-1B-Instruct auxiliary model scores tokens by mean hidden-state $L_2$ norm over its last third of layers.  Apply smooth-span segmentation (threshold 0.4, maximum 16 tokens) and completion at 0.7 coverage; map source-text intervals to target tokens and deterministically truncate/fill to the exact budget. \\
\bottomrule
\end{tabular}
\end{table*}

\paragraph{A privileged reference tests localization, not deployability.}
The \emph{clean-state oracle} ranks positions by all-layer KV $L_2$ deviation between Shift and a fully recomputed edited cache.  It then applies ordinary sparse refresh.  It requires clean recomputation to obtain its signal and is neither an optimal-support oracle nor a task-accuracy upper bound. KVEraser instead learns a different state-editing operator; a compatible trained editor checkpoint is not included in this allocator comparison \citep{li2026kveraser}.

\subsection{Datasets and answer scoring}
\label{app:common96-core-protocol}

\paragraph{Controlled NIAH.}
All five systems evaluate the same 96 controlled same-key deletion cases: Qwen3-8B, Qwen3-30B-A3B, Llama-3.3-70B-Instruct, GLM-4.5-Air, and Qwen3-235B-A22B-Instruct-2507.  Models retain their native tokenizer and chat rendering.  Deletion changes the required ordered list of numeric values; success requires exact equality between emitted and gold lists, preserving order and multiplicity.  The full-answer parser extracts digits after explicit response/channel boundaries and removal of marked thinking blocks; ordinary newlines do not truncate an answer.  It does not deduplicate, accept a gold substring, or search for a favorable answer using the gold list. The success rate is $100k/96$ for $k$ correct answers, including cases already correct under Shift.  Lexical addressability in this construction makes it a controlled repair test, not a low-overlap reasoning benchmark.

\paragraph{Natural QA.}
The panel selects 100 sources from each of HotpotQA, 2WikiMultiHopQA, MuSiQue, and IIRC before observing model outcomes.  Each source supplies before, between, and after-support placements of the same inserted non-supporting span; deleting that span restores the same clean document.  Thus every model has 1,200 placement cases, 300 per dataset. The low-budget comparison reports Qwen3-30B-A3B and Llama-3.3-70B with five allocators at 2\% and 5\%, plus Shift, Replay, and Re-prefill.  Decoding is greedy with an EOS-aware 64-token limit and no reasoning mode for Qwen.  Placements remain paired by source; they are not independent questions.  Moving the insertion changes suffix length as well as evidence geometry, unlike the separate matched construction.

Natural-QA results use answer-only EM with dataset-specific normalization. HotpotQA uses normalized string EM.  2WikiMultiHopQA additionally maximizes over answer aliases and demonyms; MuSiQue uses the answer and its supplied aliases. IIRC uses the DROP evaluator's numeric-aware normalization and span matching. All scorers use the same extracted answers. We evaluate answer correctness rather than supporting-fact or joint scores. Pooled EM uses unrounded correct counts across the four datasets; rounding is applied only for presentation (Appendix~\ref{app:natural-detail}).

\paragraph{Reference implementation.}
The evaluation harness loads model weights in BF16 with automatic device placement and uses greedy decoding. The main NIAH sweep, including Dirty, Shift, Replay, and Re-prefill, uses a 32-token limit without early EOS termination. Natural QA uses the EOS-aware 64-token limit stated above. Cache repair uses the sequential executor and unfused positional correction. The model forwards retain each architecture's native rotary configuration; the common intervention is defined in token positions rather than a shared tokenizer vocabulary. These choices specify the evaluated reference implementation, not an optimized serving backend.

\section{Complete budget--quality results}
\label{app:budget-frontiers}

\subsection{Five-model NIAH allocation sweep}

Tables~\ref{tab:revision-niah-qwen3-8b}--\ref{tab:revision-niah-qwen3-235b} report success rates (\%) for every allocation rule at every measured budget, with the same 96-case denominator throughout. References occupy a separate column. Values are point estimates without paired uncertainty or repeated-run variance.

The complete sweep exposes both useful operating points and reversals. SUR reaches 98.96\% success on Qwen3-8B at 5\% recomputation, 100\% on Llama-3.3-70B at 10\%, and 100\% on GLM-4.5-Air at 2\%.  It is not uniformly best: on Qwen3-235B-A22B, BM25 obtains 98.96\% at 2\%, versus SUR's 97.92\%; on Llama-3.3-70B, InfoFlow-Motivated obtains 100\% at 5\%, versus SUR's 98.96\%. The full budget curves show saturation, nonmonotonicity, and occasional scores above the measured Replay endpoint.

\begin{table*}[tp]
\caption{\textbf{Qwen3-8B: complete NIAH allocation sweep.} Success rate (\%, $\uparrow$) on the same $N=96$ cases. Column headings give the recomputed suffix budget; references occupy their own column. The clean-state oracle uses privileged information.}
\label{tab:revision-niah-qwen3-8b}
\centering\small
\setlength{\tabcolsep}{3.2pt}
\begin{tabular}{@{}lrrrrrrr@{}}
\toprule
Method & Reference & 2\% & 5\% & 10\% & 20\% & 40\% & 80\% \\
\midrule
Dirty & 0.00 & -- & -- & -- & -- & -- & -- \\
Shift & 0.00 & -- & -- & -- & -- & -- & -- \\
Replay & 98.96 & -- & -- & -- & -- & -- & -- \\
Re-prefill & 98.96 & -- & -- & -- & -- & -- & -- \\
\midrule
\textbf{SUR (ours)} & -- & 89.58 & 98.96 & 98.96 & 98.96 & 98.96 & 98.96 \\
Dense tokens (same score) & -- & 7.29 & 14.58 & 28.13 & 58.33 & 91.67 & 97.92 \\
Max-relevance contiguous & -- & 75.00 & 86.46 & 84.38 & 81.25 & 88.54 & 98.96 \\
BM25 (same units) & -- & 82.29 & 94.79 & 97.92 & 98.96 & 98.96 & 98.96 \\
Random semantic units & -- & 2.08 & 5.21 & 11.46 & 28.13 & 48.96 & 80.21 \\
ProphetKV-Adapted & -- & 77.08 & 97.92 & 100.00 & 100.00 & 98.96 & 98.96 \\
InfoFlow-Motivated & -- & 82.29 & 97.92 & 100.00 & 98.96 & 98.96 & 98.96 \\
Cache-Craft-Adapted & -- & 77.08 & 93.75 & 97.92 & 98.96 & 98.96 & 98.96 \\
CacheBlend-Adapted & -- & 0.00 & 4.17 & 9.38 & 22.92 & 53.13 & 94.79 \\
A\textsuperscript{3}-Adapted & -- & 1.04 & 1.04 & 13.54 & 38.54 & 81.25 & 98.96 \\
LegoLink-Adapted & -- & 1.04 & 3.13 & 7.29 & 18.75 & 44.79 & 76.04 \\
CacheTune-Adapted & -- & 1.04 & 2.08 & 4.17 & 11.46 & 27.08 & 92.71 \\
SpecCache-Adapted & -- & 0.00 & 0.00 & 1.04 & 2.08 & 22.92 & 88.54 \\
\midrule
Clean-state oracle & -- & 83.33 & 95.83 & 98.96 & 98.96 & 98.96 & 98.96 \\
\bottomrule
\end{tabular}
\end{table*}

\begin{table*}[tp]
\caption{\textbf{Qwen3-30B-A3B: complete NIAH allocation sweep.} Success rate (\%, $\uparrow$) on the same $N=96$ cases. Column headings give the recomputed suffix budget; references occupy their own column. The clean-state oracle uses privileged information.}
\label{tab:revision-niah-qwen3-30b}
\centering\small
\setlength{\tabcolsep}{3.2pt}
\begin{tabular}{@{}lrrrrrrr@{}}
\toprule
Method & Reference & 2\% & 5\% & 10\% & 20\% & 40\% & 80\% \\
\midrule
Dirty & 1.04 & -- & -- & -- & -- & -- & -- \\
Shift & 13.54 & -- & -- & -- & -- & -- & -- \\
Replay & 95.83 & -- & -- & -- & -- & -- & -- \\
Re-prefill & 98.96 & -- & -- & -- & -- & -- & -- \\
\midrule
\textbf{SUR (ours)} & -- & 95.83 & 96.88 & 95.83 & 96.88 & 95.83 & 96.88 \\
Dense tokens (same score) & -- & 89.58 & 91.67 & 92.71 & 94.79 & 92.71 & 97.92 \\
Max-relevance contiguous & -- & 84.38 & 79.17 & 79.17 & 78.13 & 77.08 & 95.83 \\
BM25 (same units) & -- & 90.63 & 94.79 & 94.79 & 95.83 & 95.83 & 95.83 \\
Random semantic units & -- & 16.67 & 17.71 & 27.08 & 36.46 & 52.08 & 80.21 \\
ProphetKV-Adapted & -- & 91.67 & 94.79 & 97.92 & 95.83 & 96.88 & 95.83 \\
InfoFlow-Motivated & -- & 91.67 & 96.88 & 96.88 & 95.83 & 96.88 & 96.88 \\
Cache-Craft-Adapted & -- & 92.71 & 95.83 & 94.79 & 96.88 & 95.83 & 95.83 \\
CacheBlend-Adapted & -- & 71.88 & 79.17 & 83.33 & 86.46 & 96.88 & 95.83 \\
A\textsuperscript{3}-Adapted & -- & 28.13 & 37.50 & 61.46 & 86.46 & 92.71 & 94.79 \\
LegoLink-Adapted & -- & 16.67 & 18.75 & 28.13 & 34.38 & 55.21 & 81.25 \\
CacheTune-Adapted & -- & 12.50 & 22.92 & 39.58 & 48.96 & 69.79 & 91.67 \\
SpecCache-Adapted & -- & 14.58 & 14.58 & 20.83 & 42.71 & 81.25 & 89.58 \\
\midrule
Clean-state oracle & -- & 95.83 & 95.83 & 95.83 & 96.88 & 96.88 & 95.83 \\
\bottomrule
\end{tabular}
\end{table*}

\begin{table*}[tp]
\caption{\textbf{Llama-3.3-70B: complete NIAH allocation sweep.} Success rate (\%, $\uparrow$) on the same $N=96$ cases. Column headings give the recomputed suffix budget; references occupy their own column. The clean-state oracle uses privileged information.}
\label{tab:revision-niah-llama3-3-70b}
\centering\small
\setlength{\tabcolsep}{3.2pt}
\begin{tabular}{@{}lrrrrrrr@{}}
\toprule
Method & Reference & 2\% & 5\% & 10\% & 20\% & 40\% & 80\% \\
\midrule
Dirty & 0.00 & -- & -- & -- & -- & -- & -- \\
Shift & 1.04 & -- & -- & -- & -- & -- & -- \\
Replay & 100.00 & -- & -- & -- & -- & -- & -- \\
Re-prefill & 100.00 & -- & -- & -- & -- & -- & -- \\
\midrule
\textbf{SUR (ours)} & -- & 90.63 & 98.96 & 100.00 & 100.00 & 100.00 & 100.00 \\
Dense tokens (same score) & -- & 7.29 & 8.33 & 7.29 & 16.67 & 29.17 & 89.58 \\
Max-relevance contiguous & -- & 65.63 & 72.92 & 80.21 & 81.25 & 85.42 & 100.00 \\
BM25 (same units) & -- & 75.00 & 94.79 & 98.96 & 100.00 & 100.00 & 100.00 \\
Random semantic units & -- & 2.08 & 5.21 & 12.50 & 26.04 & 45.83 & 79.17 \\
ProphetKV-Adapted & -- & 85.42 & 98.96 & 100.00 & 100.00 & 100.00 & 100.00 \\
InfoFlow-Motivated & -- & 92.71 & 100.00 & 100.00 & 100.00 & 100.00 & 100.00 \\
Cache-Craft-Adapted & -- & 52.08 & 82.29 & 95.83 & 100.00 & 100.00 & 100.00 \\
CacheBlend-Adapted & -- & 1.04 & 1.04 & 3.13 & 7.29 & 25.00 & 75.00 \\
A\textsuperscript{3}-Adapted & -- & 1.04 & 2.08 & 8.33 & 20.83 & 54.17 & 83.33 \\
LegoLink-Adapted & -- & 1.04 & 3.13 & 9.38 & 19.79 & 43.75 & 78.13 \\
CacheTune-Adapted & -- & 1.04 & 1.04 & 1.04 & 1.04 & 1.04 & 1.04 \\
SpecCache-Adapted & -- & 1.04 & 1.04 & 1.04 & 1.04 & 5.21 & 79.17 \\
\midrule
Clean-state oracle & -- & 60.42 & 93.75 & 98.96 & 100.00 & 100.00 & 100.00 \\
\bottomrule
\end{tabular}
\end{table*}

\begin{table*}[tp]
\caption{\textbf{GLM-4.5-Air: complete NIAH allocation sweep.} Success rate (\%, $\uparrow$) on the same $N=96$ cases. Column headings give the recomputed suffix budget; references occupy their own column. The clean-state oracle uses privileged information.}
\label{tab:revision-niah-glm4-5-air}
\centering\small
\setlength{\tabcolsep}{3.2pt}
\begin{tabular}{@{}lrrrrrrr@{}}
\toprule
Method & Reference & 2\% & 5\% & 10\% & 20\% & 40\% & 80\% \\
\midrule
Dirty & 0.00 & -- & -- & -- & -- & -- & -- \\
Shift & 11.46 & -- & -- & -- & -- & -- & -- \\
Replay & 100.00 & -- & -- & -- & -- & -- & -- \\
Re-prefill & 100.00 & -- & -- & -- & -- & -- & -- \\
\midrule
\textbf{SUR (ours)} & -- & 100.00 & 100.00 & 100.00 & 100.00 & 100.00 & 100.00 \\
Dense tokens (same score) & -- & 43.75 & 57.29 & 67.71 & 85.42 & 91.67 & 100.00 \\
Max-relevance contiguous & -- & 89.58 & 88.54 & 83.33 & 83.33 & 89.58 & 100.00 \\
BM25 (same units) & -- & 94.79 & 100.00 & 100.00 & 100.00 & 100.00 & 100.00 \\
Random semantic units & -- & 18.75 & 19.79 & 22.92 & 36.46 & 46.88 & 80.21 \\
ProphetKV-Adapted & -- & 93.75 & 100.00 & 100.00 & 100.00 & 100.00 & 100.00 \\
InfoFlow-Motivated & -- & 95.83 & 98.96 & 100.00 & 100.00 & 100.00 & 100.00 \\
Cache-Craft-Adapted & -- & 91.67 & 97.92 & 100.00 & 100.00 & 100.00 & 100.00 \\
CacheBlend-Adapted & -- & 21.88 & 28.13 & 28.13 & 43.75 & 66.67 & 97.92 \\
A\textsuperscript{3}-Adapted & -- & 18.75 & 19.79 & 27.08 & 50.00 & 78.13 & 100.00 \\
LegoLink-Adapted & -- & 16.67 & 22.92 & 28.13 & 39.58 & 56.25 & 85.42 \\
CacheTune-Adapted & -- & 18.75 & 13.54 & 14.58 & 13.54 & 18.75 & 39.58 \\
SpecCache-Adapted & -- & 16.67 & 16.67 & 19.79 & 38.54 & 61.46 & 94.79 \\
\midrule
Clean-state oracle & -- & 91.67 & 100.00 & 100.00 & 100.00 & 100.00 & 100.00 \\
\bottomrule
\end{tabular}
\end{table*}

\begin{table*}[tp]
\caption{\textbf{Qwen3-235B-A22B: complete NIAH allocation sweep.} Success rate (\%, $\uparrow$) on the same $N=96$ cases. Column headings give the recomputed suffix budget; references occupy their own column. The clean-state oracle uses privileged information.}
\label{tab:revision-niah-qwen3-235b}
\centering\small
\setlength{\tabcolsep}{3.2pt}
\begin{tabular}{@{}lrrrrrrr@{}}
\toprule
Method & Reference & 2\% & 5\% & 10\% & 20\% & 40\% & 80\% \\
\midrule
Dirty & 0.00 & -- & -- & -- & -- & -- & -- \\
Shift & 68.75 & -- & -- & -- & -- & -- & -- \\
Replay & 97.92 & -- & -- & -- & -- & -- & -- \\
Re-prefill & 100.00 & -- & -- & -- & -- & -- & -- \\
\midrule
\textbf{SUR (ours)} & -- & 97.92 & 100.00 & 100.00 & 100.00 & 100.00 & 100.00 \\
Dense tokens (same score) & -- & 83.33 & 91.67 & 96.88 & 95.83 & 97.92 & 97.92 \\
Max-relevance contiguous & -- & 72.92 & 77.08 & 83.33 & 89.58 & 97.92 & 100.00 \\
BM25 (same units) & -- & 98.96 & 100.00 & 100.00 & 100.00 & 100.00 & 100.00 \\
Random semantic units & -- & 78.13 & 83.33 & 81.25 & 78.13 & 76.04 & 86.46 \\
ProphetKV-Adapted & -- & 92.71 & 95.83 & 97.92 & 95.83 & 96.88 & 95.83 \\
InfoFlow-Motivated & -- & 95.83 & 96.88 & 96.88 & 100.00 & 100.00 & 100.00 \\
Cache-Craft-Adapted & -- & 85.42 & 96.88 & 100.00 & 100.00 & 100.00 & 100.00 \\
CacheBlend-Adapted & -- & 78.13 & 83.33 & 89.58 & 93.75 & 97.92 & 96.88 \\
A\textsuperscript{3}-Adapted & -- & 76.04 & 80.21 & 87.50 & 93.75 & 94.79 & 100.00 \\
LegoLink-Adapted & -- & 77.08 & 82.29 & 86.46 & 95.83 & 98.96 & 100.00 \\
CacheTune-Adapted & -- & 71.88 & 77.08 & 79.17 & 80.21 & 83.33 & 90.63 \\
SpecCache-Adapted & -- & 78.13 & 79.17 & 84.38 & 87.50 & 91.67 & 95.83 \\
\midrule
Clean-state oracle & -- & 93.75 & 100.00 & 100.00 & 100.00 & 100.00 & 100.00 \\
\bottomrule
\end{tabular}
\end{table*}

\subsection{Natural-QA quality and latency}
\label{app:natural-detail}

Tables~\ref{tab:revision-naturalqa-qwen3-30b} and~\ref{tab:revision-naturalqa-llama3-3-70b} give the complete low-budget comparison. Tables~\ref{tab:naturalqa-counts-qwen3-30b} and~\ref{tab:naturalqa-counts-llama3-3-70b} additionally report each dataset. The tables pair exact-match accuracy with post-edit latency at each budget, retaining the dense reference paths as separate operating points.

\paragraph{Exact-match aggregation.}
Each dataset contributes 300 cases from 100 source questions and three placements. Let $k_{m,b,d}$ be the number of correct answers for model $m$, repair condition $b$, and dataset $d$. All conditions use all four datasets, giving the pooled exact-match score
\begin{equation}
 \operatorname{EM}_{m,b}
 =100\,\frac{\sum_d k_{m,b,d}}{\sum_d n_{m,b,d}}
 =100\,\frac{\sum_d k_{m,b,d}}{1200}.
\end{equation}
Dataset EM is $100k_{m,b,d}/300$, with $n_{m,b,d}=300$. Equal dataset sizes make the unrounded macro average identical to the pooled score. Rounding is applied only for presentation.

\paragraph{Latency measurement.}
Post-edit latency sums signal acquisition, budget allocation, position correction, the sparse/replay/re-prefill forward, and query-plus-answer decoding. Sparse refresh includes prefix copying and K/V concatenation. Initial dirty-context prefill and harness cache cloning are excluded; Re-prefill's edited-context rebuild is included. We report the mean over 1,200 cases, with one timed trial per case and no warmup. All methods for a case run on the same worker using the shared sequential executor. SUR selects a median of one interval in all four model--budget conditions; the packed refresh study separately measures the benefit of fragmented execution.

\paragraph{Hardware and runtime.}
Natural-QA workers allocate one GPU for Qwen3-30B-A3B and four for Llama-3.3-70B from an H100/H200 resource pool. Models use BF16 weights and automatic device placement. The runtime uses PyTorch 2.9.1.7 with CUDA 12.9; Qwen uses automatic SDPA backend selection, while Llama disables the cuDNN SDPA backend. Hardware is shared across methods within each case, but the resource pool permits different GPU types across cases.

\begin{table*}[tp]
\caption{\textbf{Qwen3-30B-A3B: Natural QA at low budgets.} Exact match (\%) on 1,200 cases (300 per dataset) and mean post-edit latency, including selection and decoding.}
\label{tab:revision-naturalqa-qwen3-30b}
\centering\small
\setlength{\tabcolsep}{8pt}
\begin{tabular}{@{}lrrr@{}}
\toprule
Method & Budget & EM (\%) $\uparrow$ & Mean latency (s) $\downarrow$ \\
\midrule
\textbf{SUR (ours)} & 2\% & 42.58 & 1.663 \\
\textbf{SUR (ours)} & 5\% & 43.42 & 2.039 \\
BM25 (same units) & 2\% & 41.08 & 1.688 \\
BM25 (same units) & 5\% & 41.50 & 2.130 \\
CacheBlend-Adapted & 2\% & 41.42 & 3.486 \\
CacheBlend-Adapted & 5\% & 41.75 & 6.306 \\
InfoFlow-Motivated & 2\% & 41.83 & 3.965 \\
InfoFlow-Motivated & 5\% & 42.58 & 6.751 \\
ProphetKV-Adapted & 2\% & 40.92 & 4.076 \\
ProphetKV-Adapted & 5\% & 41.17 & 6.853 \\
\midrule
Shift & -- & 41.67 & 1.173 \\
Replay & -- & 43.00 & 2.164 \\
Re-prefill & -- & 43.33 & 2.360 \\
\bottomrule
\end{tabular}
\end{table*}

\begin{table*}[tp]
\caption{\textbf{Llama-3.3-70B: Natural QA at low budgets.} Exact match (\%) on 1,200 cases (300 per dataset) and mean post-edit latency, including selection and decoding.}
\label{tab:revision-naturalqa-llama3-3-70b}
\centering\small
\setlength{\tabcolsep}{8pt}
\begin{tabular}{@{}lrrr@{}}
\toprule
Method & Budget & EM (\%) $\uparrow$ & Mean latency (s) $\downarrow$ \\
\midrule
\textbf{SUR (ours)} & 2\% & 56.67 & 0.500 \\
\textbf{SUR (ours)} & 5\% & 56.92 & 0.535 \\
BM25 (same units) & 2\% & 54.75 & 0.490 \\
BM25 (same units) & 5\% & 55.08 & 0.552 \\
CacheBlend-Adapted & 2\% & 54.67 & 1.572 \\
CacheBlend-Adapted & 5\% & 55.17 & 2.759 \\
InfoFlow-Motivated & 2\% & 55.42 & 1.536 \\
InfoFlow-Motivated & 5\% & 56.00 & 2.859 \\
ProphetKV-Adapted & 2\% & 55.33 & 1.571 \\
ProphetKV-Adapted & 5\% & 55.67 & 2.803 \\
\midrule
Shift & -- & 54.67 & 0.384 \\
Replay & -- & 56.83 & 0.635 \\
Re-prefill & -- & 57.00 & 0.766 \\
\bottomrule
\end{tabular}
\end{table*}

\subsection{Paired uncertainty and error transitions}
\label{app:naturalqa-paired-ci}

Confidence intervals use a dataset-stratified, source-paired bootstrap. Each replicate draws 100 source questions with replacement from each dataset, retaining all three placements and all methods for every drawn question. We use 10,000 replicates with seed 20260928 and linear-interpolated 2.5th and 97.5th percentiles. Table~\ref{tab:naturalqa-paired-ci} reports pointwise 95\% intervals for $\Delta\mathrm{EM}=\mathrm{EM}_{\sur{}}-\mathrm{EM}_{\mathrm{reference}}$, in percentage points, without a multiple-comparison adjustment.

At 5\%, \sur{} improves on Shift by 1.75 points on Qwen3-30B ($[1.12,2.24]$) and 2.25 points on Llama-3.3-70B ($[1.56,3.21]$). The corresponding differences from Re-prefill are $+0.08$ ($[-0.53,+0.63]$) and $-0.08$ ($[-0.48,+0.40]$). The 2\% Qwen3-30B point trades more quality for speed: its difference from Re-prefill is $-0.75$ points ($[-1.14,-0.02]$). These comparisons make the role of the budget explicit: more recomputation narrows the accuracy gap to the dense reference.

The error-transition analysis separates repairs from regressions. At 5\%, \sur{} corrects Shift errors on 2.50\% of Qwen3-30B cases and 3.50\% of Llama-3.3-70B cases, while introducing errors on 0.75\% and 1.25\%, respectively.

\begin{table}[htbp]
\centering\small
\caption{\textbf{Paired exact-match differences on Natural QA.} SUR minus each reference in percentage points, with 95\% paired bootstrap confidence intervals. Each model is evaluated on 400 source questions and 1,200 placement cases.}
\label{tab:naturalqa-paired-ci}
\begin{tabular}{@{}clrc@{}}
\toprule
Budget & Reference & $\Delta$EM & 95\% CI \\
\midrule
\multicolumn{4}{@{}l}{\textbf{Qwen3-30B-A3B}} \\
2\% & Shift & $+0.92$ & $[+0.28, +1.56]$ \\
2\% & Replay & $-0.42$ & $[-0.95, +0.13]$ \\
2\% & Re-prefill & $-0.75$ & $[-1.14, -0.02]$ \\
5\% & Shift & $+1.75$ & $[+1.12, +2.24]$ \\
5\% & Replay & $+0.42$ & $[-0.15, +0.88]$ \\
5\% & Re-prefill & $+0.08$ & $[-0.53, +0.63]$ \\
\midrule
\multicolumn{4}{@{}l}{\textbf{Llama-3.3-70B}} \\
2\% & Shift & $+2.00$ & $[+1.28, +2.93]$ \\
2\% & Replay & $-0.17$ & $[-0.78, +0.46]$ \\
2\% & Re-prefill & $-0.33$ & $[-1.15, +0.88]$ \\
5\% & Shift & $+2.25$ & $[+1.56, +3.21]$ \\
5\% & Replay & $+0.08$ & $[-0.65, +0.71]$ \\
5\% & Re-prefill & $-0.08$ & $[-0.48, +0.40]$ \\
\bottomrule
\end{tabular}
\end{table}


\begin{table*}[tp]
\centering\small
\caption{\textbf{Qwen3-30B-A3B: performance by dataset.} Exact match (\%, $\uparrow$) across four datasets, with 300 cases per dataset and 1,200 pooled cases.}
\label{tab:naturalqa-counts-qwen3-30b}
\setlength{\tabcolsep}{5pt}
\begin{tabular}{@{}lrrrrrr@{}}
\toprule
Method & Budget & 2Wiki & HotpotQA & IIRC & MuSiQue & Pooled \\
\midrule
\textbf{SUR (ours)} & 2\% & 52.33 & 47.33 & 36.67 & 34.00 & 42.58 \\
\textbf{SUR (ours)} & 5\% & 53.00 & 48.33 & 37.67 & 34.67 & 43.42 \\
BM25 (same units) & 2\% & 52.00 & 46.33 & 35.00 & 31.00 & 41.08 \\
BM25 (same units) & 5\% & 52.33 & 46.33 & 35.67 & 31.67 & 41.50 \\
CacheBlend-Adapted & 2\% & 52.67 & 46.33 & 35.33 & 31.33 & 41.42 \\
CacheBlend-Adapted & 5\% & 52.67 & 47.00 & 35.67 & 31.67 & 41.75 \\
InfoFlow-Motivated & 2\% & 52.33 & 46.67 & 36.00 & 32.33 & 41.83 \\
InfoFlow-Motivated & 5\% & 52.67 & 47.33 & 37.00 & 33.33 & 42.58 \\
ProphetKV-Adapted & 2\% & 52.33 & 46.00 & 35.00 & 30.33 & 40.92 \\
ProphetKV-Adapted & 5\% & 52.33 & 46.33 & 35.33 & 30.67 & 41.17 \\
\midrule
Shift & -- & 52.33 & 46.67 & 36.00 & 31.67 & 41.67 \\
Replay & -- & 52.67 & 48.00 & 37.67 & 33.67 & 43.00 \\
Re-prefill & -- & 53.00 & 47.67 & 38.00 & 34.67 & 43.33 \\
\bottomrule
\end{tabular}
\end{table*}

\begin{table*}[tp]
\centering\small
\caption{\textbf{Llama-3.3-70B: performance by dataset.} Exact match (\%, $\uparrow$) across four datasets, with 300 cases per dataset and 1,200 pooled cases.}
\label{tab:naturalqa-counts-llama3-3-70b}
\setlength{\tabcolsep}{5pt}
\begin{tabular}{@{}lrrrrrr@{}}
\toprule
Method & Budget & 2Wiki & HotpotQA & IIRC & MuSiQue & Pooled \\
\midrule
\textbf{SUR (ours)} & 2\% & 65.00 & 59.33 & 53.33 & 49.00 & 56.67 \\
\textbf{SUR (ours)} & 5\% & 65.00 & 59.67 & 53.67 & 49.33 & 56.92 \\
BM25 (same units) & 2\% & 61.00 & 58.00 & 51.33 & 48.67 & 54.75 \\
BM25 (same units) & 5\% & 61.67 & 58.67 & 51.33 & 48.67 & 55.08 \\
CacheBlend-Adapted & 2\% & 61.33 & 58.33 & 51.00 & 48.00 & 54.67 \\
CacheBlend-Adapted & 5\% & 62.00 & 58.67 & 51.33 & 48.67 & 55.17 \\
InfoFlow-Motivated & 2\% & 63.00 & 58.67 & 52.00 & 48.00 & 55.42 \\
InfoFlow-Motivated & 5\% & 64.00 & 59.33 & 52.33 & 48.33 & 56.00 \\
ProphetKV-Adapted & 2\% & 62.00 & 59.00 & 51.67 & 48.67 & 55.33 \\
ProphetKV-Adapted & 5\% & 63.00 & 59.33 & 51.33 & 49.00 & 55.67 \\
\midrule
Shift & -- & 61.33 & 57.67 & 51.33 & 48.33 & 54.67 \\
Replay & -- & 64.67 & 61.00 & 51.67 & 50.00 & 56.83 \\
Re-prefill & -- & 66.00 & 60.33 & 52.67 & 49.00 & 57.00 \\
\bottomrule
\end{tabular}
\end{table*}

\section{Ablations and design analysis}
\label{app:design-analysis}

The central question is how to allocate a fixed repair budget around a query-relevance signal. The 96-case comparisons hold cases, exact selected-token allowance, and all-layer refresh operator fixed. We first compare independent tokens, one contiguous window, and SUR using identical token scores, then examine query direction and alternative relevance rules. The full budget curves appear in Appendix~\ref{app:budget-frontiers}.

\subsection{Allocation geometry at a fixed relevance score}

Table~\ref{tab:measured-grouping-ablation} concentrates on the 2\% and 5\% operating points, where differences in budget allocation are largest. Independent top-token selection takes the $B$ highest SUR token scores $r_i$. The single-window control instead maximizes summed $r_i$ over all width-$B$ contiguous suffix windows. SUR allocates the same allowance across ranked units. This sequence compares support organization under a shared relevance signal, including a score-optimized single-region control.

\begin{table*}[t]
\centering
\caption{\textbf{Allocation geometry on Erasing-NIAH.} Exact match (\%, $\uparrow$) over the same 96 cases per model, at 2\% and 5\% retained-suffix budgets. Token selection and the single window use the same token scores as SUR; the window maximizes their sum over all width-$B$ suffix windows. BM25 and random units retain SUR's partition. Bold denotes the best value within each column.}
\label{tab:measured-grouping-ablation}
\begingroup
\fontsize{9.5}{11}\selectfont
\setlength{\tabcolsep}{2.5pt}
\renewcommand{\arraystretch}{1.12}
\begin{tabular}{@{}l*{10}{r}@{}}
\toprule
 & \multicolumn{2}{c}{Qwen3-8B} & \multicolumn{2}{c}{Qwen3-30B} & \multicolumn{2}{c}{Llama-70B} & \multicolumn{2}{c}{GLM-Air} & \multicolumn{2}{c}{Qwen3-235B} \\
\cmidrule(lr){2-3} \cmidrule(lr){4-5} \cmidrule(lr){6-7} \cmidrule(lr){8-9} \cmidrule(lr){10-11}
Allocation rule & 2\% & 5\% & 2\% & 5\% & 2\% & 5\% & 2\% & 5\% & 2\% & 5\% \\
\midrule
\textbf{SUR} & \textbf{89.6} & \textbf{99.0} & \textbf{95.8} & \textbf{96.9} & \textbf{90.6} & \textbf{99.0} & \textbf{100.0} & \textbf{100.0} & 97.9 & \textbf{100.0} \\
Same-score tokens & 7.3 & 14.6 & 89.6 & 91.7 & 7.3 & 8.3 & 43.8 & 57.3 & 83.3 & 91.7 \\
Single relevance window & 75.0 & 86.5 & 84.4 & 79.2 & 65.6 & 72.9 & 89.6 & 88.5 & 72.9 & 77.1 \\
Same-unit BM25 & 82.3 & 94.8 & 90.6 & 94.8 & 75.0 & 94.8 & 94.8 & \textbf{100.0} & \textbf{99.0} & \textbf{100.0} \\
Random units & 2.1 & 5.2 & 16.7 & 17.7 & 2.1 & 5.2 & 18.8 & 19.8 & 78.1 & 83.3 \\
\bottomrule
\end{tabular}
\endgroup
\end{table*}

\begin{table*}[t]
\centering
\caption{\textbf{Alternative allocation policies.} Exact match (\%, $\uparrow$) on the same 96 cases and budgets. Adapted/motivated methods use the common repair operator, not their native systems. Their score-driven token allocation differs from SUR's unit allocation; this is a comparison of complete selection rules, not an isolated scoring ablation.}
\label{tab:measured-signal-comparison}
\begingroup
\fontsize{9.5}{11}\selectfont
\setlength{\tabcolsep}{2.5pt}
\renewcommand{\arraystretch}{1.12}
\begin{tabular}{@{}l*{10}{r}@{}}
\toprule
 & \multicolumn{2}{c}{Qwen3-8B} & \multicolumn{2}{c}{Qwen3-30B} & \multicolumn{2}{c}{Llama-70B} & \multicolumn{2}{c}{GLM-Air} & \multicolumn{2}{c}{Qwen3-235B} \\
\cmidrule(lr){2-3} \cmidrule(lr){4-5} \cmidrule(lr){6-7} \cmidrule(lr){8-9} \cmidrule(lr){10-11}
Allocation rule & 2\% & 5\% & 2\% & 5\% & 2\% & 5\% & 2\% & 5\% & 2\% & 5\% \\
\midrule
\textbf{SUR} & \textbf{89.6} & \textbf{99.0} & \textbf{95.8} & \textbf{96.9} & 90.6 & 99.0 & \textbf{100.0} & \textbf{100.0} & \textbf{97.9} & \textbf{100.0} \\
ProphetKV-Adapted & 77.1 & 97.9 & 91.7 & 94.8 & 85.4 & 99.0 & 93.8 & \textbf{100.0} & 92.7 & 95.8 \\
InfoFlow-Motivated & 82.3 & 97.9 & 91.7 & \textbf{96.9} & \textbf{92.7} & \textbf{100.0} & 95.8 & 99.0 & 95.8 & 96.9 \\
Cache-Craft-Adapted & 77.1 & 93.8 & 92.7 & 95.8 & 52.1 & 82.3 & 91.7 & 97.9 & 85.4 & 96.9 \\
CacheBlend-Adapted & 0.0 & 4.2 & 71.9 & 79.2 & 1.0 & 1.0 & 21.9 & 28.1 & 78.1 & 83.3 \\
\bottomrule
\end{tabular}
\endgroup
\end{table*}

\paragraph{Coherence helps, and multiple relevant regions help further.}
At 2\%, the token--window--SUR sequence yields 7.3\%, 75.0\%, and 89.6\% EM on Qwen3-8B, and 7.3\%, 65.6\%, and 90.6\% on Llama-3.3-70B. Thus concentrating computation within a score-optimized window recovers much of the gap, while allocation across query-relevant regions improves further. SUR exceeds the single-window control at 2\% on all five systems. At 5\%, SUR and dense token selection attain 98.96\% versus 14.58\% success on Qwen3-8B and 98.96\% versus 8.33\% on Llama-3.3-70B. The corresponding contrast is smaller on Qwen3-30B-A3B: 96.88\% versus 91.67\%. The magnitude depends on the model; the measured advantage belongs to the allocation policy, which changes support geometry and interval count.

\paragraph{Grouping alone is insufficient.}
Random units retain SUR's partition but remove query ranking, reaching only 2.1\% EM on Qwen3-8B and Llama-3.3-70B at 2\%. SUR exceeds this control on all five systems. The useful combination is coherent allocation and query direction, rather than grouping alone.

\subsection{Inexpensive relevance versus attention and state signals}

\paragraph{Cheap relevance signals support effective region allocation.}
BM25 retains SUR's unit partition and allocator, replacing its relevance rule with lexical scoring. It reaches 82.29\% EM on Qwen3-8B at 2\%. SUR exceeds BM25 at 2\% on four models, whereas Qwen3-235B-A22B favors BM25 (98.96\% versus 97.92\%).  This contrast tests the complete relevance rule, not embedding similarity or the identical-token bonus in isolation. Both inexpensive rules can direct effective repair.

\paragraph{Richer signals are evaluated as complete policies.}
The attention- and state-derived allocators in Table~\ref{tab:measured-signal-comparison} extend the comparison to richer signals. ProphetKV-Adapted averages query attention over heads and layers; InfoFlow-Motivated uses an attention-norm signal. Cache-Craft-Adapted scores exposure to the deleted span, while CacheBlend-Adapted scores an early-layer K/V deviation probe. Appendix~\ref{app:deletion-matrix-protocol} gives the precise definitions and signal-acquisition costs.

At 2\%, SUR exceeds each of these four policies on Qwen3-8B, Qwen3-30B-A3B, GLM-4.5-Air, and Qwen3-235B-A22B. On Llama-3.3-70B, however, InfoFlow-Motivated reaches 92.71\%, compared with SUR's 90.63\%, and reaches 100\% at 5\%, compared with 98.96\%. These results favor combining a low-cost signal with region-level allocation and show that richer signals can improve particular operating points. These policies select individual score-ranked tokens, so the comparison varies both scoring and support construction; it does not isolate signal quality.

\subsection{Matched-token evaluation and design variants}

\paragraph{A larger matched-token holdout.}
A separate Qwen3-8B experiment evaluates SUR V1 and token-query selection on a frozen 599-case holdout. Table~\ref{tab:historical-holdout} compares matched mean selected-token counts: region-level selection improved EM by 9.85 percentage points at the smaller support and by 1.84 points at the larger support. The benefit was concentrated at 1K--2K contexts; both selectors approached saturation from 4K onward, with one SUR regression at 32K. This larger cohort supports the grouping observation under a matched-token protocol. The SUR V1 allocator and budget definition differ from the main suffix-budget sweep, so the results are analyzed separately.

\begin{table}[t]
\centering
\caption{\textbf{Matched-token evaluation.} Qwen3-8B exact match (\%) over 599 cases. SUR V1 and token-query selection use matched mean selected-token counts, evaluated separately from the suffix-budget sweep.}
\label{tab:historical-holdout}
\small
\setlength{\tabcolsep}{5pt}
\begin{tabular}{@{}rrr@{}}
\toprule
Mean selected tokens & SUR V1 & Query \\
\midrule
144.6 & 97.5 & 87.6 \\
359.2 & 99.7 & 97.8 \\
\bottomrule
\end{tabular}
\end{table}

\paragraph{Joint scoring and window variants.}
A separate 24-case Qwen3-8B comparison tested top-four aggregation with a forward partial window against mean aggregation with a prefix partial unit. The former also used a suffix-relative allowance, whereas the latter used a whole-prompt allowance.  The variants tied in observed EM from the nominal 5\% points upward; their target-reaching points selected approximately 244 versus 519 tokens and formed 7.3 versus 20.1 intervals on average. The top-four variant had lower agreement with the clean greedy continuation. The comparison illustrates a quality--support tradeoff. Because three design choices change together and the budget convention differs from the main sweep, it does not isolate any individual component.

\paragraph{Scope of the design evidence.}
Without matched fixed-block and multiple-local-window controls, these results do not establish a unique advantage for linguistic boundaries. The identical-token bonus, top-four aggregation, forward decay, and unit-length thresholds are held at the settings in Appendix~\ref{app:deletion-matrix-protocol}; their individual sensitivity remains a direction for further analysis.

\section{Matched evidence-location experiment}
\label{app:matched-extended}

The direct allocation controls in Appendix~\ref{app:design-analysis} test how \sur{} spends a repair budget. This supplementary experiment asks an earlier question: whether moving task evidence downstream changes repair demand when edit geometry is fixed. On a separately constructed panel, the gap between full suffix Replay and Shift measures the value of contextual recomputation at each evidence placement; it is not an estimate of \sur{}'s gain.

\subsection{Construction and controls}

The candidate pool contains the same 400 natural-QA sources.  All annotated support documents form an untruncated evidence bundle.  A control bundle contains one to four native non-support documents, rejecting documents with a normalized answer alias or support title.  Whole natural text chunks and a fixed neutral padding family extend both bundles to an identical token width.  A source is excluded if exact construction is infeasible; eligibility does not depend on a model outcome.

The bundles exchange slots immediately before and after a fixed 256-token deletion.  Within each tail condition, the exchange holds edit index, deletion width, slot width, retained-suffix length, text multiset, query, and answer target constant.  A neutral-tail extension outside the slots crosses this exchange to form four cells: upstream/short, downstream/short, upstream/long, and downstream/long.  Thus the side contrast fixes suffix length, while the tail contrast jointly changes filler, stale-state volume, and evidence-to-query distance.  The evidence exchange also changes support order and recency; it is not an isolated intervention on a unique semantic mediator.

Geometry selection uses only tokenization and searches a fixed list of candidates before model execution. Their (slot width, edit index, short suffix, long suffix) values are $(512,2048,1024,4096)$, then $(768,3072,1536,4096)$, then $(1024,4096,2048,4096)$ tokens.  Literal support text is not truncated to force a match.  Exact boundary tokenization, absence of protected answer strings in padding, and absence of special-token IDs in inserted content are checked before execution.  The common panel contains 275 sources: 79 from 2WikiMultiHopQA, 64 from HotpotQA, 64 from IIRC, and 68 from MuSiQue. Qwen3-8B, Llama-3.1-8B-Instruct, and Qwen3-30B-A3B evaluate all four cells, giving 825 source--model units.  Additional sparse diagnostics use an outcome-blind 100-source subset, with 25 sources per dataset.

\subsection{Estimands and uncertainty}

For evidence side $z\in\{U,D\}$ and tail $\ell\in\{S,L\}$, let $G_{z,\ell}$ denote the mean repair advantage.  For the preregistered likelihood endpoint this is gold-answer $\mathrm{NLL}_{\minimal{}}- \mathrm{NLL}_{\exactrepair{}}$; for task EM it is $\mathrm{EM}_{\exactrepair{}}-\mathrm{EM}_{\minimal{}}$. The three factorial contrasts are
\begin{align}
 \tau_{\mathrm{side}}&=\tfrac12\bigl[(G_{D,S}-G_{U,S})\nonumber\\
 &\qquad +(G_{D,L}-G_{U,L})\bigr],\\
 \tau_{\mathrm{tail}}&=\tfrac12\bigl[(G_{U,L}-G_{U,S})\nonumber\\
 &\qquad +(G_{D,L}-G_{D,S})\bigr],\\
 \tau_{\mathrm{interaction}}&=(G_{D,L}-G_{U,L})\nonumber\\
 &\qquad -(G_{D,S}-G_{U,S}).
\end{align}
Canonical-answer NLL averages teacher-forced negative log probabilities over answer tokens and includes every complete source.  The EM secondary uses the stratum answered correctly by the clean reference in all four cells: 300 of 825 source--model units, comprising 106, 95, and 99 units for the three models.  This common stratum is not conditioned on Shift failure.

We average paired contrasts within source, then within dataset, giving equal weight to the four datasets and fixed models.  The 10,000-replicate bootstrap resamples sources within dataset and carries all associated cells and model rows together.  Intervals use the 2.5th and 97.5th percentile. The pooled secondary family (NLL tail and interaction; EM side, tail, and interaction) uses Holm correction. Because the primary NLL result is null, model-specific intervals are descriptive rather than confirmatory.

\begin{table*}[tp]
\caption{\textbf{Matched factorial effects.} EM effects are in percentage points and NLL effects in nats.  The preregistered NLL evidence-side interval includes zero, whereas the EM secondary is positive after correction.}
\label{tab:matched-factorial-effects}
\centering
\small
\begin{tabular}{llrrr}
\toprule
Endpoint & Contrast & Estimate & 95\% CI & Holm $p$ \\
\midrule
NLL & Evidence side & $-0.039$ & [$-0.099$, 0.021] & --- \\
 & Neutral tail & $-0.031$ & [$-0.066$, 0.004] & 0.308 \\
 & Interaction & $-0.001$ & [$-0.075$, 0.068] & 1.000 \\
\addlinespace
EM & Evidence side & 5.99 & [2.87, 9.26] & 0.002 \\
 & Neutral tail & 0.49 & [$-1.6$, 2.6] & 1.000 \\
 & Interaction & $-1.2$ & [$-4.8$, 2.2] & 1.000 \\
\bottomrule
\end{tabular}
\end{table*}

The likelihood side contrast is heterogeneous across datasets: $-0.155$, $-0.045$, $0.022$, and $0.021$ nats on 2WikiMultiHopQA, HotpotQA, IIRC, and MuSiQue. The positive task-EM effect does not imply a general improvement in answer likelihood.

\subsection{Localized refresh and continuation fidelity}

A fixed-slot control refreshes the complete downstream exchange slot. Its token budget and interval are identical whether the slot holds evidence or non-supporting content.  Table~\ref{tab:matched-clean-effects} compares its side effect with complete suffix replay on output distribution and continuation agreement.  Similar point estimates indicate that the exchanged region carries much of this measured side effect; they do not establish equivalence of the repairs or identify a minimum sufficient support.

\begin{table*}[tp]
\caption{\textbf{Evidence-side effects on clean-reference fidelity.} Positive effects mean a larger repair benefit over Shift with downstream evidence.  Entries are estimates and source-bootstrap 95\% intervals. TV and JSD use reductions in distance; agreement effects are in percentage points (pp).}
\label{tab:matched-clean-effects}
\centering
\small
\setlength{\tabcolsep}{4pt}
\begin{tabular}{lrr}
\toprule
Endpoint & Complete replay & Fixed-slot refresh \\
\midrule
First-token TV & .0696 [.0560, .0832] & .0709 [.0571, .0849] \\
First-token JSD & .0311 [.0238, .0385] & .0315 [.0240, .0391] \\
Argmax agreement (pp) & 8.37 [6.27, 10.49] & 8.17 [5.96, 10.42] \\
Sequence agreement (pp) & 11.77 [9.27, 14.25] & 11.45 [8.81, 14.16] \\
\bottomrule
\end{tabular}
\end{table*}

Task accuracy, continuation fidelity, and K/V tensor identity remain different endpoints.  The fixed-slot NLL side contrast is $-0.031$ nats ([$-0.093$, $0.032$]); the localized fidelity result does not establish a positive likelihood effect either.

\section{Numerical execution-path controls}
\label{app:numerical-systems}
\label{app:numerical-timing-contract}

Replay and Re-prefill agree on the intended edited history but can disagree numerically under BF16. A separate repeat control measures this effect on 40 sources with four geometry cells each, giving 160 cells per model.  Its models are Qwen3-8B, Llama-3.1-8B, and Qwen3-30B-A3B; it is not a repeat of the five-model NIAH or two-model Natural-QA sweep.

Each cell has three decodes from a shared Re-prefill cache and three independently built caches per reference path.  For categorical outputs $v_1,v_2,v_3$, within-condition agreement is
\begin{equation}
 A_{\mathrm{within}}(v)=\frac13\sum_{1\leq i<j\leq3}\mathbf{1}[v_i=v_j].
 \label{eq:repeat-within-agreement}
\end{equation}
Cross-path agreement averages over all nine Replay--Re-prefill pairs. The unclipped excess disagreement is $A_{\mathrm{within}}(\full{})-A_{\mathrm{cross}}(\exactrepair{},\full{})$. For NLL, the cell-level analogue subtracts the median of three absolute within-Re-prefill differences from the median of nine absolute cross-path differences.  Table~\ref{tab:matched-numerical-floor} summarizes these execution-path references.  Exact-logit agreement means byte equality of the recorded float32 first-logit vectors, not tolerance-based closeness.

\begin{table*}[tp]
\caption{\textbf{Numerical-repeat controls on a separate 160-cell cohort per model.} Categorical entries are disagreement percentages; excess is in percentage points.  NLL entries summarize absolute differences.  The reference model Llama-3.1-8B differs from the main sweep's Llama-3.3-70B.}
\label{tab:matched-numerical-floor}
\centering\small
\setlength{\tabcolsep}{4pt}
\begin{tabular}{@{}llrrrr@{}}
\toprule
Model & Output & Shared (\%) & Re-prefill (\%) & Cross-path (\%) & Excess (pp) \\
\midrule
Qwen3-8B & Sequence & 0.42 & 0.83 & 2.29 & 1.46 \\
 & Exact logits & 0.00 & 0.00 & 100.00 & 100.00 \\
 & Argmax & 0.00 & 0.00 & 1.88 & 1.88 \\
\addlinespace
Llama-3.1-8B & Sequence & 2.29 & 1.67 & 3.68 & 2.01 \\
 & Exact logits & 0.00 & 0.00 & 100.00 & 100.00 \\
 & Argmax & 0.00 & 0.00 & 0.62 & 0.62 \\
\addlinespace
Qwen3-30B-A3B & Sequence & 0.42 & 0.83 & 10.69 & 9.86 \\
 & Exact logits & 0.00 & 0.00 & 100.00 & 100.00 \\
 & Argmax & 0.00 & 0.00 & 8.75 & 8.75 \\
\bottomrule
\end{tabular}
\par\medskip
\begin{tabular}{@{}lrrr@{}}
\toprule
Model & Within Re-prefill $|\Delta\mathrm{NLL}|$ & Cross-path $|\Delta\mathrm{NLL}|$ & Excess \\
\midrule
Qwen3-8B & 0.0000 & 0.0785 & 0.0785 \\
Llama-3.1-8B & 0.0000 & 0.0278 & 0.0278 \\
Qwen3-30B-A3B & 0.0000 & 0.2037 & 0.2037 \\
\bottomrule
\end{tabular}
\end{table*}

These controls distinguish decoder reuse, independent cache builds, and cross-path computation.  They are not a deterministic-kernel or FP32 reference and do not define a universal threshold for semantic error. In particular, 100\% exact-logit disagreement can coexist with far smaller argmax and sequence disagreement. Measured Replay accuracy is therefore an operational comparator, not a mathematical upper bound on sparse accuracy.

\paragraph{Distributional recovery is not KV-tensor recovery.}
Where distributional diagnostics are reported, first-token total variation is $\mathrm{TV}(p,q)=\tfrac12\sum_v|p(v)-q(v)|$, and normalized recovery is
\begin{equation}
 \operatorname{NR}_{\mathrm{TV}}(a)
 =1-\frac{\mathrm{TV}(p_a,p_{\mathrm{clean}})}
          {\mathrm{TV}(p_{\mathrm{dirty}},p_{\mathrm{clean}})}.
 \label{eq:normalized-tv-recovery}
\end{equation}
This ratio is formed per case, is undefined when its denominator is zero, and is not clipped; a reported median is a median of defined per-case ratios.  It measures first-token output-distribution recovery, not similarity of the KV tensors. It complements exact-match accuracy and does not supply uncertainty for the main aggregate results.

\section{Reproducibility}
\label{app:reproducibility}

The experiment configuration specifies model/tokenizer revisions, native chat rendering, decoding limits, the case roster, and the repair operator. Case identifiers preserve the common NIAH roster and source-level Natural-QA pairing.  Sparse support is checked for uniqueness, suffix membership, and the exact budget before execution; case records retain selected positions, realized intervals, and component timings.

Tables and figures use the same result series. Validation checks model and method coverage, exact token allowances, per-dataset denominators, and integer-count pooling before rounding. Natural-QA conditions each contain four 300-case datasets, with $0\leq k\leq n$ for every correct count. Paired differences satisfy $k_{\sur{}}-k_{\mathrm{reference}}=k_{\mathrm{repaired}}-k_{\mathrm{regressed}}$.

Natural-QA confidence intervals preserve source-question pairing across placements (Appendix~\ref{app:naturalqa-paired-ci}). The evidence-location and numerical experiments use their own matched cohorts, described in Appendices~\ref{app:matched-extended} and~\ref{app:numerical-systems}. Execution configurations specify the timed stages, warmups, repetitions, and hardware separately for task-level latency and refresh microbenchmarks.

\section{Packed execution: implementation and controls}
\label{app:kernel-execution}

\subsection{Scaling of budgeted recomputation}
\label{app:recompute-scaling}

The budget reduces both the number of token-local forwards and the number of attention query rows. Let $P=|p|$, $S=|s|$, $N=P+S$, $A\subseteq\{1,\ldots,S\}$, and $B=|A|$. For $L$ transformer layers, the ratio of refreshed token-layer positions to full suffix Replay is
\begin{equation}
 \frac{LB}{LS}=\frac{B}{S}\approx\frac{b}{100}.
 \label{eq:recompute-volume}
\end{equation}
For a 10,000-token retained suffix, a 2--5\% budget refreshes 200--500 positions through every layer, reusing the remaining 9,500--9,800. For a fixed dense transformer architecture of hidden width $d$ and depth $L$, an idealized useful-arithmetic model is
\begin{align}
 F_{\mathrm{Replay}}
 &=\Theta\!\left(L\!\left[Sd^2+d\!\sum_{i=1}^{S}(P+i)\right]\right),\\
 F_{\mathrm{refresh}}
 &=\Theta\!\left(L\!\left[Bd^2+d\!\sum_{i\in A}(P+i)\right]\right).
 \label{eq:refresh-arithmetic}
\end{align}
The first term covers Q/K/V and output projections and the MLP, whose per-token arithmetic is fixed by the architecture. Its reduction factor is exactly $B/S$. The second counts legal causal attention pairs: each selected query still reads its complete legal prefix, but unselected positions no longer issue new queries.

Full suffix Replay evaluates $SP+S(S+1)/2$ causal pairs per head. For any support of size $B$, the largest possible selected-pair count is obtained by choosing the final $B$ suffix positions. Consequently,
\begin{equation}
 \begin{aligned}
 \frac{\sum_{i\in A}(P+i)}{SP+S(S+1)/2}
 &\leq\frac{B(P+S)-B(B-1)/2}{SP+S(S+1)/2}\\
 &\leq\min\!\left\{1,\frac{2B}{S}\right\}.
 \end{aligned}
 \label{eq:attention-pair-bound}
\end{equation}
For a 10K suffix, a 2\% budget therefore refreshes 200 positions and uses at most 4\% of Replay's causal query--key pairs; a 5\% budget refreshes 500 positions and uses at most 10\%. Token-local projection/MLP arithmetic uses 2\% and 5\%, respectively. These are architectural operation counts for the recomputation stage. Selection, Shift, cache traffic, masked-pair computation, tile padding, and invocation overhead are accounted for separately in execution timing.

Packing addresses a second cost: repeated model traversals. Let $R$ be the number of selected intervals, $H$ the fixed overhead per traversal, and $C_{\mathrm{interval}}(A)$ and $C_{\mathrm{packed}}(A)$ the support-dependent compute times:
\begin{equation}
 \begin{aligned}
 T_{\mathrm{interval}}&\approx RH+C_{\mathrm{interval}}(A),\\
 T_{\mathrm{packed}}&\approx H+C_{\mathrm{packed}}(A).
 \end{aligned}
 \label{eq:packing-cost}
\end{equation}
At fixed $A$, interval-wise execution incurs $R$ model traversals, while packing incurs one. The support-dependent terms can differ because packing changes matrix shapes and kernel utilization. Thus the model identifies an overhead-amortization regime with gains on the order of $R$; the measured 16.1$\times$ is an empirical result, not an exact latency prediction from a token-count ratio. Growing the suffix increases the number of avoided forwards, while the actual time saved depends on this balance of arithmetic and execution overhead.

\paragraph{From refresh to complete-answer cost.}
\label{app:latency-accounting}
Refresh time covers all-layer projections, attention, and MLPs, together with cache copying, planning, and uploads. With serial preparation,
\begin{equation}
 \begin{aligned}
 T_{\mathrm{recovery}}&=T_{\mathrm{select}}+T_{\mathrm{shift}}
                         +T_{\mathrm{refresh}},\\
 T_{\mathrm{post\text{-}edit}}&=T_{\mathrm{recovery}}+T_{\mathrm{answer}}.
 \end{aligned}
 \label{eq:latency-accounting}
\end{equation}
Selection includes signal acquisition and unit construction; answering includes query processing and decoding. Overlapped recovery uses its actual critical path. Refresh measures execution efficiency; recovery includes selection and position repair; post-edit latency measures the cost of delivering an answer.

\subsection{Mixed-cache recurrence and implementation}

The packed executor accepts the same fixed support as interval-wise repair. Let $U=(u_1,\ldots,u_B)$ list the selected positions in increasing order within the edited context, using zero-based indexing: suffix index $i\in A$ maps to $|p|+i-1$. The input shifted cache $\mathcal C_0$ is not modified: one complete output clone is made, selected token IDs are gathered once, and their embeddings initialize a $B$-row hidden-state buffer that remains compact through all layers.

At layer $\ell$, the model's native projections, Q/K normalization, and rotary embeddings produce fresh Q/K/V for the packed rows using their positions $U$. Let $K^+,V^+$ denote the fresh K/V and $K^0,V^0$ the layer's entries in $\mathcal C_0$. Before attention, the output cache is updated to
\begin{equation}
 (K_j,V_j)=\begin{cases}
 (K^+_r,V^+_r), & j=u_r,\\
 (K^{0}_j,V^{0}_j), & j\notin U.
 \end{cases}
\end{equation}
For packed query row $r$, the attention head output is
\begin{equation}
 O_r=\operatorname{softmax}\!\left(
 \frac{Q_r K_{0:u_r}^{\top}}{\sqrt{d_h}}\right)V_{0:u_r},
 \label{eq:indexed-attention}
\end{equation}
where $d_h$ is the head dimension and $0{:}u_r$ includes position $u_r$. Each selected query reads both fresh K/V at causally available selected positions and reused K/V in the intervening gaps; future positions remain masked. The native output projection, residuals, normalization, and MLP produce the packed hidden states for the next layer.

With fixed positions, causal attention, and token-local operations outside attention, the layer-major schedule preserves the interval-major recurrence in real arithmetic. Both schedules start from the same embeddings and shifted cache. If the selected hidden inputs agree at a layer, their fresh Q/K/V agree, and each query reads the same fresh and reused causal entries; the layer outputs therefore agree as well. Writing future selected K/V early does not affect this argument because those entries are masked. Changing GEMM shapes and reduction order can nevertheless change BF16 outputs, so the scheduling equivalence is not a claim of bitwise equality (Appendix~\ref{app:kernel-numerics}).

The default backend uses native SDPA. The adapter is implemented for Transformers 4.57.6 native dense Llama/Qwen3 in evaluation mode, full attention, default fixed RoPE, batch size one, and a single device. It does not accept MoE, sliding-window attention, offload, dynamic RoPE, or custom model forwards. This scope is narrower than the five-model quality comparison. The reference implementation also provides a one-shot sparse path, included as a control for the benefit of packing already available without the new adapter.

\paragraph{Indexed attention and writes.}
The optional Triton attention processes 16 packed query rows and 64 K/V positions per tile, with four warps and two stages. Grouped-query heads use the original head mapping. Online softmax and tiled accumulation follow established IO-aware attention techniques \citep{dao2022flashattention}. The kernel keeps probabilities in FP32 and explicitly uses TF32x3 products for $P V$, rather than rounding probabilities to BF16/FP16. This is not IEEE-exact FP32. The kernel requires NVIDIA SM80+, FP16/BF16 inputs, and head dimension 64 or 128. An optional indexed scatter writes fresh K and V in one launch. Scatter and attention remain separate, ordered launches on the same stream; there is no cross-program global barrier assumption. Torch indexed writes are the default because they are faster in the measured complete-refresh path.

\paragraph{Reusable execution plan.}
For a query tile with first and last selected positions $u_{\min},u_{\max}$, the plan records the complete-prefix and total K/V tile counts:
\begin{align}
 n_{\mathrm{full}}&=\lfloor(u_{\min}+1)/64\rfloor,\\
 n_{\mathrm{total}}&=\lceil(u_{\max}+1)/64\rceil.
\end{align}
Complete-prefix tiles are visible to every row; remaining tiles apply the per-element test $j\leq u_r$. Widely separated selected rows can require multiple partially visible tiles. Positions and tile bounds are prepared and uploaded once per refresh, then reused across layers. Given sorted $U$, plan construction is $O(B)$, without a materialized $B\times N$ mask for the custom backend. Attention still processes $M(U)=\sum_{u\in U}(u+1)$ causal pairs per head, plus tile padding; it is not $B^2$ attention. The full output clone costs $O(LNH_{\mathrm{KV}}d_h)$ memory traffic, and selected writes cost $O(LBH_{\mathrm{KV}}d_h)$.

\subsection{Repeated execution measurements}

The execution benchmarks use one Qwen3-8B checkpoint (36 layers, hidden size 4096, 32 query and 8 K/V heads, head dimension 128) on an H100 80GB HBM3, BF16, batch size one. The environment uses CUDA 12.9, PyTorch 2.9.1, Transformers 4.57.6, and Triton 3.5.1. A random-token prompt loses eight tokens and undergoes the reference eager Shift before refresh timing. The selected budget is $B=64$; $R=1,4,16$ controls interval count. Within each length, $M(U)$ is identical across these geometries: 65,568 at 2K and 262,176 at 8K. Tile occupancy can still differ.

Each backend receives ten warmups and twenty randomly interleaved repeated measurements. Table~\ref{tab:kernel-full} reports synchronized wall-time medians, including output cloning, host planning/upload, checks, native module execution, and any per-call mask construction. Prefill, selector, Shift, later query/answer decoding, and first compilation are excluded. There is no CUDA Graph. Original one-shot and direct sequential controls retain their distinct cache-assembly costs. Packed SDPA, FlexAttention \citep{pytorch2025flexattention}, and custom attention use the same independent-output adapter. Thus their comparison isolates backend choice more closely than comparison with the original sequential implementation.

\begin{table*}[t]
\centering
\caption{\textbf{Complete matched-work refresh timings (ms).} Qwen3-8B, H100, $B=64$; medians of twenty measurements after ten warmups. Direct sequential updates the cache without repeated full assembly. Custom uses Torch writes unless marked fused and remains experimental (Appendix~\ref{app:kernel-numerics}).}
\label{tab:kernel-full}
\small
\begin{tabular}{@{}rrrrrrrrr@{}}
\toprule
$N$ & $R$ & Sequential & Direct seq. & One-shot & SDPA & Custom & + fused & Flex \\
\midrule
2048 & 1 & 33.051 & 31.465 & 34.593 & 32.485 & 32.389 & 34.169 & 53.008 \\
2048 & 4 & 131.290 & 121.456 & 34.624 & 32.672 & 32.470 & 34.063 & 56.661 \\
2048 & 16 & 530.068 & 485.378 & 34.942 & 32.952 & 32.775 & 34.588 & 56.858 \\
8192 & 1 & 38.848 & 32.571 & 42.126 & 34.860 & 33.372 & 35.228 & 95.613 \\
8192 & 4 & 142.898 & 120.679 & 41.198 & 34.733 & 32.235 & 34.051 & 144.759 \\
8192 & 16 & 552.498 & 494.826 & 42.186 & 34.757 & 33.375 & 35.161 & 144.376 \\
\bottomrule
\end{tabular}
\end{table*}

\paragraph{Attention-only comparison.}
Table~\ref{tab:kernel-attention} holds BF16 Q/K/V resident, with the same 64 queries, 32 query heads, 8 K/V heads, and head dimension 128. Ten warmups precede thirty interleaved samples. Timing includes Python dispatch, output allocation, and conversion to the same contiguous output layout; planning/upload, mask construction, compilation, cache clones, scatter, and other transformer operators are excluded. Every output is checked against the same explicit FP64 absolute-causal oracle. The pinned Hugging Face SDPA wrapper repeats K/V heads on its explicit-mask path. Direct SDPA with grouped-query attention is therefore a stronger microbaseline: custom is faster only at 8K with one interval (1.18$\times$), and slower in the other five cases. FlexAttention is not tuned for these workloads. These wrapper-inclusive timings do not establish a general ranking of attention algorithms or pure device-kernel latency.

\begin{table*}[t]
\centering
\caption{\textbf{Attention-only median time ($\mu$s).} Resident BF16 operands; thirty measurements after ten warmups. Direct SDPA enables GQA rather than explicitly repeating K/V heads.}
\label{tab:kernel-attention}
\small
\begin{tabular}{@{}rrrrrr@{}}
\toprule
$N$ & $R$ & HF SDPA & Direct SDPA+GQA & Compiled Flex & Custom \\
\midrule
2048 & 1 & 167.3 & 124.9 & 1019.2 & 135.1 \\
2048 & 4 & 167.6 & 121.8 & 1106.8 & 148.3 \\
2048 & 16 & 164.8 & 122.2 & 1108.3 & 147.6 \\
8192 & 1 & 411.1 & 273.9 & 2145.7 & 232.5 \\
8192 & 4 & 401.8 & 266.6 & 3477.7 & 293.7 \\
8192 & 16 & 402.8 & 277.9 & 3468.8 & 295.7 \\
\bottomrule
\end{tabular}
\end{table*}

\paragraph{Preparation overlap.}
A reusable CPU worker constructs semantic units while the GPU computes relevance. After scores return, CPU budget allocation and host planning overlap the original eager Shift. The selector, scores, retokenization, and selected support are unchanged. A separate synthetic natural-text input has $N=2412$, $B=64$, and $R=5$ after deleting 15 tokens. Using the same experimental custom-attention/fused-scatter backend on both sides, ten warmups and twenty repeated samples give 56.738\,ms for serial preparation and 53.790\,ms with overlap. This includes scoring, unit construction, selection, plan/upload, Shift, and refresh, but excludes prefill, initial prompt/query tokenization, worker creation, compilation, and subsequent decoding. Support, diagnostics, shifted/output cache, and query logits agree bitwise across the two paths. This is one input, not the Natural-QA evaluation or an average end-to-end generation result.

\subsection{Timing examples on actual QA supports}
\label{app:qa-recovery-timing}

To connect the scheduling result to actual allocator outputs, we time one fixed input from each QA dataset on the same Qwen3-8B/H100 configuration. We select the lexicographically first source ID and its before-support placement, independently of outputs or timing. Edited lengths are 1,319, 1,606, 1,421, and 2,769 tokens. Each method receives two warmups and five interleaved measured repetitions; Table~\ref{tab:qa-recovery-timing} gives means with equal weight per input.

\begin{table}[t]
\centering
\caption{\textbf{Real-input recovery timing examples (ms).} Four fixed QA inputs on Qwen3-8B/H100. Refresh includes the executor's planning and output ownership; recovery additionally includes relevance, unit construction, allocation, and Shift. Reference operators have no SUR selection stage.}
\label{tab:qa-recovery-timing}
\small
\begin{tabular}{@{}lrr@{}}
\toprule
Method & Refresh/operator & Recovery \\
\midrule
\sur{} 2\%, interval-wise & 32.87 & 51.39 \\
\sur{} 2\%, packed & 31.98 & 50.43 \\
\sur{} 5\%, interval-wise & 48.99 & 67.25 \\
\sur{} 5\%, packed & 31.98 & 50.71 \\
\midrule
Full suffix Replay & 48.57 & 48.57 \\
Clean Re-prefill & 60.25 & 60.25 \\
\bottomrule
\end{tabular}
\end{table}

At 5\%, packing yields 1.53$\times$ faster refresh and 1.33$\times$ faster recovery than interval-wise execution. Two inputs select two intervals; their refresh times fall from 65.65/65.37 to 32.14/32.10\,ms, approximately 2.04$\times$. The other supports occupy one interval. Across these short inputs, packed 5\% recovery is 1.19$\times$ faster than Re-prefill and 0.96$\times$ as fast as Replay. This illustrates both the benefit of amortizing multiple traversals and the fixed costs that remain when recomputation is already small. The probes measure synchronized cache recovery only; they do not generate or score answers and do not alter the main task-quality or complete-answer latency results.

\subsection{Numerical validation}
\label{app:kernel-numerics}

Packing preserves the mathematical recurrence but not low-precision execution order. On the six refresh configurations, both packed SDPA and custom attention preserve the input and all unselected K/V bits, and match the sequential path's two-token query argmaxes. For fragmented supports, packed SDPA's query logit relative $L_2$ differences range from 0.67\% to 1.89\%. Custom selected-K/V maximum absolute differences range from 0.5 to 0.875; the largest layerwise relative $L_2$ difference is 1.67\%. Query-logit relative $L_2$ differences are 0.54--0.97\%. These are numerical probes, not task-quality certification.

The indexed attention unit tests use an explicit FP64 oracle with fixed BF16 tolerances $\mathrm{atol}=0.002$, $\mathrm{rtol}=0.02$ (FP16: 0.00025 and 0.003), covering causal gaps, tails, GQA, strides, metadata changes, and non-default streams. With FP32 probabilities retained, the custom kernel passes this local test on operands from all 36 layers of the 2K checkpoint run; the largest error/tolerance ratio is 0.190. The full H100 regression suite passes 830 tests; 81 GPU memory-sanitizer checks report no memory errors. These counts are software tests, not additional task samples.

A separate complete-model control promotes identical BF16 weights and buffers, and the shifted starting cache, to FP32 with PyTorch TF32 disabled. Its engineering budget requires each selected-K/V and query-logit max/RMS error to remain within twice the native low-precision packed-SDPA error. An earlier strict pointwise integration check had failed, and this diagnostic budget was fixed before the subsequent checkpoint controls; it is not an accuracy guarantee. At 8K/$R=16$, the custom path passes, with maximum ratio 1.617. At 2K/$R=16$, zero-based layer 18 V maximum error is 0.21346 versus the native path's 0.09965: a ratio of 2.142 exceeds the fixed budget. RMS and final-logit checks pass, but this failure remains. Layer-local accuracy does not establish identical accumulated trajectories. Accordingly, custom attention remains an explicit experimental option; packed SDPA is the default. The task-quality results elsewhere in this paper were not rerun on either packed backend.

\section{Mixed-Cache Error under Sparse Refresh}
\label{app:mixed_cache_error}

Sparse refresh recomputes selected positions while retaining shifted K/V states at unselected positions. Hence, even a refreshed position can attend to stale causal predecessors. We formalize this residual error for a single attention head.

Consider a refreshed position $i$ under repair support $A$. Let
\[
o_i^A = \sum_{j\le i}\alpha_{ij}^A v_j^A,
\qquad
\alpha_i^A=\operatorname{softmax}(z_i^A),
\]
where
\[
z_{ij}^A
=
\frac{\langle q_i^A,k_j^A\rangle}{\sqrt{d_h}}.
\]
Let the superscript $\star$ denote the corresponding quantities under full-suffix \textsc{Replay}, and define $\Delta q_i=q_i^A-q_i^\star$, $\Delta k_j=k_j^A-k_j^\star$, and $\Delta v_j=v_j^A-v_j^\star$.

\begin{proposition}[Residual mixed-cache error]
\label{prop:mixed_cache_error}
Assume $\|v_j^\star\|_2\le M_V$ for all $j\le i$. Then
\begin{align}
\|o_i^A-o_i^\star\|_2
\le\;&
\sum_{j\le i}\alpha_{ij}^A\|\Delta v_j\|_2
\nonumber\\
&+
2M_V\|z_i^A-z_i^\star\|_\infty .
\label{eq:mixed_attention_bound}
\end{align}
Moreover,
\begin{equation}
\begin{aligned}
&\|z_i^A-z_i^\star\|_\infty\\
&\quad\le \frac{1}{\sqrt{d_h}}\Bigl(
K_\star\|\Delta q_i\|_2\\
&\qquad\quad+
\bigl(Q_\star+\|\Delta q_i\|_2\bigr)
\max_{j\le i}\|\Delta k_j\|_2\Bigr),
\end{aligned}
\label{eq:logit_perturbation_bound}
\end{equation}
where $Q_\star=\|q_i^\star\|_2$ and $K_\star=\max_{j\le i}\|k_j^\star\|_2$.
\end{proposition}

\begin{proof}
Adding and subtracting $\sum_{j\le i}\alpha_{ij}^A v_j^\star$ gives
\begin{align}
\|o_i^A-o_i^\star\|_2
\le\;&
\sum_{j\le i}\alpha_{ij}^A\|\Delta v_j\|_2
\nonumber\\
&+
M_V\|\alpha_i^A-\alpha_i^\star\|_1.
\end{align}
Using the $2$-Lipschitz bound of softmax from $\ell_\infty$ to $\ell_1$ yields Eq.~\eqref{eq:mixed_attention_bound}.

For each $j\le i$,
\[
\begin{aligned}
&\langle q_i^A,k_j^A\rangle-\langle q_i^\star,k_j^\star\rangle\\
&\quad=\langle\Delta q_i,k_j^\star\rangle
+\langle q_i^\star,\Delta k_j\rangle\\
&\qquad+\langle\Delta q_i,\Delta k_j\rangle.
\end{aligned}
\]
Cauchy--Schwarz followed by maximization over $j\le i$ gives Eq.~\eqref{eq:logit_perturbation_bound}.
\end{proof}

\paragraph{Interpretation.}
Under the fixed-attention approximation $\alpha_i^A=\alpha_i^\star$, Proposition~\ref{prop:mixed_cache_error} reduces to
\begin{equation}
\|o_i^A-o_i^\star\|_2
\le
\sum_{j\le i}
\alpha_{ij}^\star\|\Delta v_j\|_2.
\label{eq:fixed_attention_bound}
\end{equation}
Thus, refreshing position $i$ alone does not remove errors carried by earlier stale states in its causal context. Refreshing influential predecessors can reduce this residual term when their recomputation reduces their deviation from \textsc{Replay}. The bound does not imply that semantic boundaries are optimal, nor that large attention mass necessarily produces large output error; SUR uses semantic units as a lightweight structural heuristic rather than computing a causal dependency closure.

\end{document}